\documentclass{article}
\usepackage{iclr2027_conference}
\usepackage{times}
\usepackage{amsmath,amsfonts,bm}

\def\eqref#1{equation~\ref{#1}}
\def\1{\bm{1}}

\DeclareMathAlphabet{\mathsfit}{\encodingdefault}{\sfdefault}{m}{sl}
\SetMathAlphabet{\mathsfit}{bold}{\encodingdefault}{\sfdefault}{bx}{n}

\usepackage{url}
\usepackage{amsmath,amssymb,amsthm}
\usepackage{bm,mathtools}
\usepackage{multirow,array}
\usepackage{booktabs}
\usepackage{xcolor}
\usepackage{colortbl}
\usepackage{arydshln}
\usepackage{graphicx}
\usepackage{wrapfig}
\usepackage{needspace}
\usepackage{algorithm}
\usepackage{algorithmic}
\usepackage{hyperref}
\usepackage{cleveref}

\newtheorem{theorem}{Theorem}
\newtheorem{proposition}[theorem]{Proposition}
\newtheorem{lemma}[theorem]{Lemma}
\newtheorem{remark}[theorem]{Remark}
\newtheorem{corollary}[theorem]{Corollary}

\newcommand{\method}{FidelityQ}

\title{\method: Preserving Predictive Fidelity in\\Ultra-Low-Bit Language Model Quantization}

\author{Yujie Chen$^{1}$, Weiyan Zheng$^{2}$, Yiwu Yao$^{2}$, Tianchi Hu$^{2}$, Bo Wang$^{2}$, Jiyu Guo$^{3}$, \\
{\bf Wenzhan Li$^{4}$, Jingcai Guo$^{1}$, Michele Magno$^{5}$, Haotong Qin$^{1}$\thanks{Corresponding author.}} \\
\textsuperscript{1}The Hong Kong Polytechnic University \quad \textsuperscript{2}Huawei Technologies Ltd. \\
\textsuperscript{3}The University of Hong Kong \quad \textsuperscript{4}Sun Yat-sen University \quad \textsuperscript{5}ETH Zurich \\
{\small\texttt{haotong.qin@polyu.edu.hk}}
}

\iclrfinalcopy
\begin{document}

\setcounter{footnote}{1}  % so the single \thanks prints as a dagger, not an asterisk
\maketitle
\lhead{}

% ====================================================================
\begin{abstract}
Ultra-low-bit scalar quantization offers a regular arithmetic structure for efficient deployment of large language models, but preserving their predictive behavior at extreme precision remains challenging.
With only a few reconstruction levels, accurate weight representation must translate into consistent model predictions using a small set of calibration samples.
We propose \textbf{\method}, a scalar quantization framework for preserving predictive fidelity at weight precisions down to 2 bits.
Specifically, we introduce \emph{Information-Aligned Analytic Coding} (\textbf{IAC}) for accurate quantization.
IAC shapes each weight row toward a centered symmetric source through randomized Hadamard mixing, then fits a mirrored four-level minimum-distortion quantizer to its empirical weights.
Every level set in this symmetry-aligned family decomposes exactly into 2 sign-valued bases with associated scales, combining source-adaptive reconstruction with lookup-free decoding compatible with scaled integer arithmetic.
Our analysis further reveals that residual predictive discrepancies are heavy-tailed and concentrate on tokens that the teacher predicts confidently, motivating \emph{Knowledge-Intensive Calibration} (\textbf{KIC}) for consistent quantization.
KIC ranks sequences using the tail of token-wise teacher-student divergence and allocates a limited calibration budget to the highest-scoring samples.
It then restores predictive agreement by updating only per-row reconstruction levels under teacher supervision, leaving the pretrained weights frozen.
Experiments on the LLaMA and Qwen families demonstrate the effectiveness of \method{} across weight-only and joint quantization settings.
On LLaMA-2-7B, \method{} achieves a WikiText2 perplexity of 6.62 at W2A16 and improves the prior state-of-the-art W2A4KV4 result from 8.31 to 6.96, while reducing linear-layer weight memory by 8.0$\times$ relative to FP16.  % [T:tab:main, tab:main-full, tab:e2e]
These results support accurate LLM deployment with structured ultra-low-bit quantization.
Our code is publicly available at \url{https://github.com/AI-Efficiency/FidelityQ}.
\end{abstract}

% ====================================================================
\section{Introduction}
\label{sec:intro}

Large language models (LLMs) support diverse applications, from code generation to embodied planning~\citep{codellama,palme,zhang2026using}, but their deployment remains constrained by the memory required to store and access their parameters~\citep{mobilellm,llmflash,ragstack}. Quantization reduces this burden by representing model parameters and intermediate tensors at lower numerical precision~\citep{gptq,awq,omniquant}. Reducing weight precision to 2 bits offers substantial storage savings, while scalar representations with simple scaling operations preserve a regular arithmetic structure for deployment under tight memory budgets. However, each scalar weight then has only 4 possible reconstruction values. Preserving a pretrained LLM's predictive fidelity under this constraint requires accurate reconstruction of its weights.

Existing quantization methods improve reconstruction through different representations and optimization strategies. Scalar methods compensate quantization errors or adapt channel scales while retaining simple decoding operations~\citep{gptq,awq,omniquant,dbllm}. Nevertheless, a fixed or restricted set of reconstruction values can remain poorly matched to the weight distribution at 2 bits. Codebook methods offer richer representations and competitive accuracy~\citep{quip-sharp,aqlm}, but require specialized decoding to translate their representation advantages into efficient execution. Calibration and distillation can further recover predictive quality by optimizing quantization parameters or aligning model outputs~\citep{omniquant,dbllm,efficientqat}. For example, on LLaMA-2-7B, group-wise round-to-nearest quantization increases WikiText2 perplexity from 5.47 to 431.97, whereas the published DB-LLM and QuIP\# results reduce it to 7.23 and 6.66, respectively~\citep{dbllm,efficientqat}.  % [T:tab:main, tab:main-full]

\begin{figure}[!t]
\centering
\includegraphics[width=\textwidth,trim=0 1.4bp 0 0,clip]{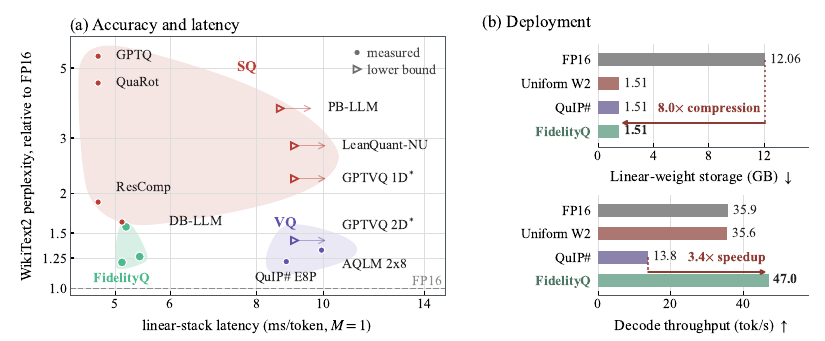}
\caption{\method{} against scalar (SQ) and codebook (VQ) quantization at two bits. (a) \method{} improves the accuracy-efficiency tradeoff: no other scalar baseline shown reaches the codebook methods' WikiText2 perplexity, and \method{} reaches it at scalar latency. $^{*}$ is 2.125 bits per weight. (b) The same representation is smaller and faster: on LLaMA-2-7B at batch size 1, its linear-layer weights take 8.0$\times$ less memory than FP16 and decode 3.4$\times$ faster than QuIP\# (Appendix~\ref{app:kernel}).}
\label{fig:tradeoff}
\end{figure}
\vspace{-4pt}

We propose \textbf{\method}, a scalar quantization framework for preserving predictive fidelity at weight precisions down to 2 bits (\cref{fig:method}). Building on classical fixed-rate scalar quantization~\citep{lloyd1982,max1960}, we view quantization as source coding under a fixed bit budget and analyze how reconstruction accuracy relates to the structure of the code. Our analysis shows that the four-level minimum-distortion solution for the centered Gaussian reference lies in a mirrored family whose reconstruction values admit an exact expression in 2 sign-valued bases and their scales. This connects source-adaptive level placement with lookup-free decoding compatible with integer arithmetic. We further consider how the remaining approximation affects model predictions. Weight reconstruction error alone does not identify the inputs on which the quantized model departs from its full-precision counterpart, motivating calibration data chosen directly from their predictive disagreement.

Specifically, we introduce \emph{Information-Aligned Analytic Coding} (\textbf{IAC}) for accurate quantization. IAC first applies randomized Hadamard mixing, following rotation-based quantization~\citep{quarot,quip-sharp}, to obtain centered, approximately symmetric, near-Gaussian weight marginals. Guided by this source geometry, it restricts the 4 reconstruction levels to two mirrored pairs and determines their magnitudes by minimizing empirical squared error once for each transformed row~\citep{lloyd1982,max1960}. Every such mirrored level set is represented exactly by 2 sign-valued bases and their amplitudes, allowing decoding through 2 scaled dot-product paths without per-element lookup. With quantized activations, the dot products use integer arithmetic while the amplitudes supply the reconstruction scales. Building on this representation, we introduce \emph{Knowledge-Intensive Calibration} (\textbf{KIC}) for consistent quantization. KIC measures token-wise divergence between the full-precision teacher and the quantized model, aggregates the largest divergences within each sequence, and selects the highest-scoring sequences under a fixed calibration budget. The selected sequences concentrate calibration on predictions that require closer agreement with the teacher. Teacher-guided calibration then updates only the per-row reconstruction levels, with the pretrained weights frozen. This procedure connects accurate weight reconstruction with targeted calibration, preserving predictive fidelity within a regular, lookup-free arithmetic structure.

We evaluate \method{} on the LLaMA and Qwen families under weight-only and joint quantization settings. On LLaMA-2-7B at W2A16, \method{} achieves a WikiText2 perplexity of 6.62, compared with 8.89 for DB-LLM~\citep{dbllm} at matched per-row granularity and 6.66 for QuIP\# under the 2048-token reevaluation of \citet{efficientqat}. At W2A4KV4 it achieves 6.96, compared with the reported 8.31 of RCP~\citep{rcp}, while retaining lookup-free decoding. On an RTX A6000 at batch size 1, its linear stack takes 5.1\,ms per token against 9.9\,ms for AQLM, and the full model reaches 47.0 tokens per second in eager mode against 35.9 for FP16 (\cref{fig:tradeoff}). These results demonstrate strong predictive performance with 2-bit scalar quantization.  % [T:tab:main, tab:main-full, fig:tradeoff, tab:e2e]

% ====================================================================
\section{Related work}
\label{sec:related}

\textbf{Post-training quantization.}
GPTQ and its extensions reduce reconstruction error through sequential compensation~\citep{gptq,gptaq,rescomp}.
AWQ and OmniQuant adapt channel scales or equivalent transformations to improve scalar quantization~\citep{awq,omniquant}.
Rotation-based methods further reshape the source: QuaRot uses fixed Hadamard rotations, while SpinQuant and FlatQuant optimize rotations or more general transformations~\citep{quarot,spinquant,flatquant}.
\method{} builds on source shaping and error compensation, keeps pretrained weights frozen during calibration, and adapts scalar reconstruction levels to the transformed weight distribution.

\textbf{Low-bit representations and decoding.}
Vector and codebook quantizers provide strong accuracy references at two bits.
QuIP\# uses a lattice codebook, while AQLM and VPTQ learn codebooks for weight reconstruction~\citep{quip-sharp,aqlm,vptq}.
Their indexed representations require decoding alongside matrix multiplication.
Scalar alternatives retain arithmetic structure: PB-LLM leaves salient weights at higher precision and BiLLM binarizes them residually, while DB-LLM reconstructs two-bit weights through two binary branches~\citep{pbllm,billm,dbllm}.
Classical scalar quantization determines reconstruction values by minimizing distortion~\citep{lloyd1982,max1960,gray-neuhoff}, which Qift applies to rotated two-bit weights as one mirrored no-zero level set shared by all rows~\citep{qift}.
IAC instead fits the levels per transformed row and expresses them exactly through two sign planes, connecting source adaptation to regular arithmetic decoding.

\textbf{Calibration and data selection.}
Quantization parameters can be refined through end-to-end optimization, as in EfficientQAT's step-size updates, or teacher supervision, as in BitDistiller and DB-LLM~\citep{efficientqat,bitdistiller,dbllm}.
Calibration data also affect compression quality~\citep{calibration-impact}; self-calibration obtains such data from model-generated text~\citep{selfcalibration}.
Broader selection methods target pretraining distributions or downstream capabilities~\citep{dsir,less}.
More closely related to compression recovery, PASER allocates instruction data according to capability degradation after pruning~\citep{paser}.
KIC selects calibration sequences using the tail of token-wise disagreement between the full-precision and quantized models, then refines only per-row reconstruction levels under teacher supervision.
Its selection criterion directly measures the predictive changes caused by quantization.

% ====================================================================
\section{Methodologies}
\label{sec:method}

We introduce \textbf{\method}, a framework that preserves predictive fidelity through a source-adaptive scalar representation and its calibration against the full-precision model.
The central design is to make reconstruction values adaptable while keeping their arithmetic decoder fixed throughout quantization and calibration.
As shown in \cref{fig:method}, \emph{Information-Aligned Analytic Coding} (IAC) constructs this representation, and \emph{Knowledge-Intensive Calibration} (KIC) refines it on sequences selected by predictive disagreement.
Both stages operate on the same reconstruction levels, connecting accurate weight reconstruction to consistent predictions within a structured low-bit representation.

\begin{figure}[t]
\centering
\includegraphics[width=\textwidth,trim=0 5bp 0 0,clip]{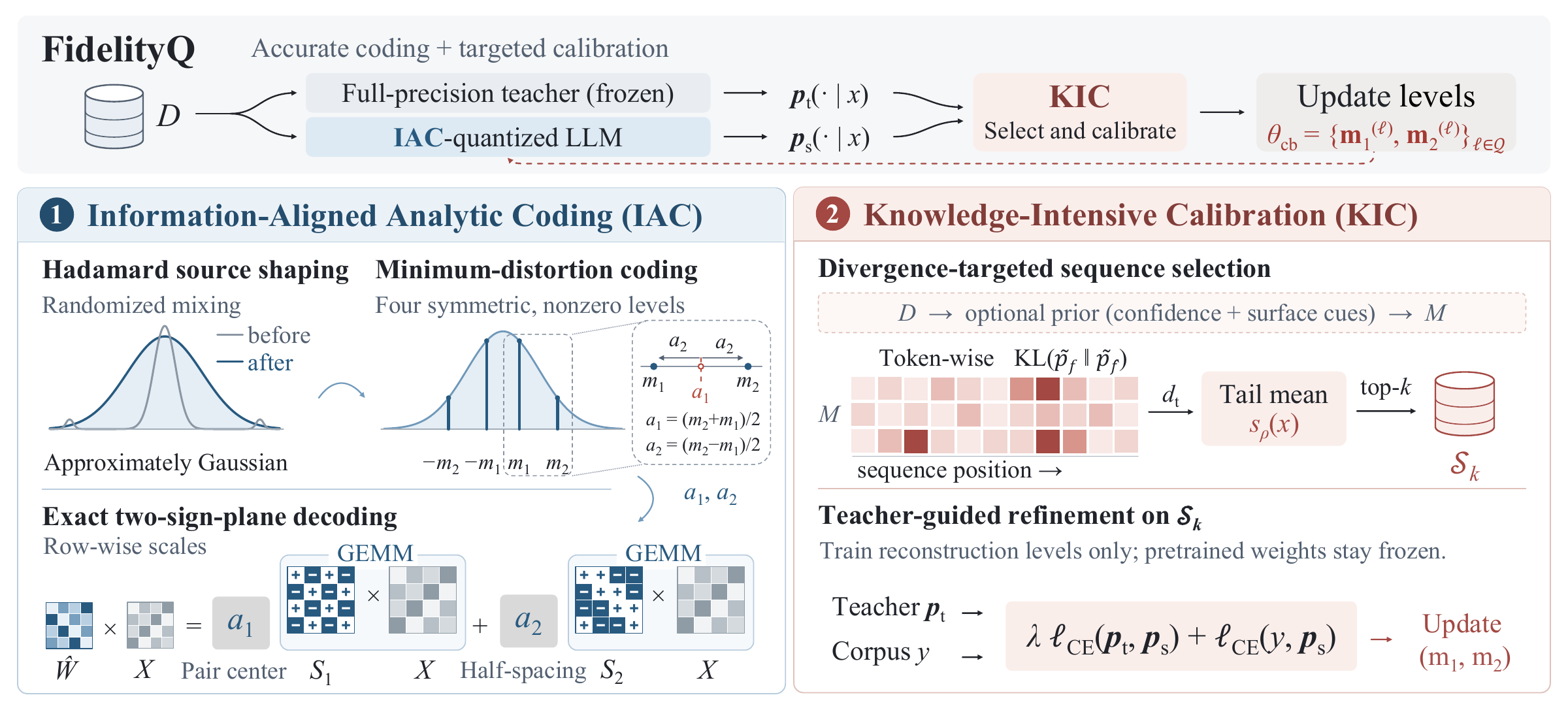}
\caption{Overview of \method{}. The top pipeline compares a frozen full-precision teacher with the IAC-quantized model, after which KIC selects calibration data and updates only the reconstruction levels. IAC shapes the source with a randomized Hadamard transform, places four symmetric nonzero levels at minimum distortion, and decodes them exactly as two sign planes with row-wise scales. KIC ranks sequences by the tail mean of token-wise divergence, keeps the top-$k$ complete sequences, and refines the levels using the teacher distribution and corpus targets.}
\label{fig:method}
\end{figure}

\subsection{From Scalar Reconstruction to Predictive Fidelity}
\label{sec:prelim}
\label{sec:rate-distortion}

At two bits, a scalar quantizer has only four reconstruction values to represent the weights.
A fixed level placement can mismatch the source distribution and incur avoidable reconstruction error.
For an empirical row distribution $p_r$ and a scalar weight $z_r\sim p_r$, the minimum-distortion objective is~\citep{lloyd1982,max1960}
\begin{equation}
\mathcal{C}_r^\star
\;=\;
\arg\min_{|\mathcal{C}_r|=4}
\mathbb{E}_{z_r\sim p_r}\!\left[
\min_{c\in\mathcal{C}_r}(z_r-c)^2
\right].
\label{eq:rate-distortion}
\end{equation}
This objective depends on how the four reconstruction values cover the source distribution.
When rare outliers determine the range of a uniform grid, the densely populated region can receive too little reconstruction precision.
The resulting weight errors enter layer computations and can propagate to model predictions.
Improving fidelity, therefore, calls for a source that is easier to quantize and reconstruction levels that match its probability mass.

\subsection{Information-Aligned Analytic Coding for Accurate Quantization}
\label{sec:iac}
\label{sec:iac-place}
\label{sec:iac-decode}

To improve scalar reconstruction while preserving arithmetic decoding, we introduce \emph{Information-Aligned Analytic Coding} (IAC).
As shown in \cref{fig:method}, IAC first shapes the weight distribution, then fits reconstruction levels and assigns quantization codes, and finally maps these codes to two sign planes for exact arithmetic decoding.

\textbf{Stage 1: Source shaping.}

To spread outlier energy across coordinates, we apply a randomized signed Hadamard transform $\boldsymbol{q}$ to weights $\boldsymbol{w}\in\mathbb{R}^{d_o\times d_i}$ and activations $\boldsymbol{x}\in\mathbb{R}^{n\times d_i}$.
With $\boldsymbol{w}'=\boldsymbol{w}\boldsymbol{q}$ and $\boldsymbol{x}'=\boldsymbol{x}\boldsymbol{q}$, orthogonality gives $\boldsymbol{x}'{\boldsymbol{w}'}^\top=\boldsymbol{x}\boldsymbol{w}^\top$.
The transform preserves the unquantized linear operation and each weight row's energy while changing the values presented to the scalar quantizer~\citep{quarot}.
For the transformed scalar source, we use a centered Gaussian reference.
This reference attains the maximum-entropy bound among centered continuous sources $z_r$ with variance $\sigma_r^2$:
\begin{equation}
h(z_r)
\;\leq\;
\frac{1}{2}\log\!\left(2\pi e\sigma_r^2\right),
\label{eq:max-entropy}
\end{equation}
with equality for $z_r\sim\mathcal{N}(0,\sigma_r^2)$.
This maximum-entropy reference characterizes the regularized marginals, providing a basis for subsequent level fitting to preserve as much of their information as possible through minimum-distortion reconstruction.

\Needspace{4\baselineskip}
\textbf{Stage 2: Adaptive level placement.}

The shaped source remains concentrated near zero, motivating nonuniform level placement to reduce reconstruction error~\citep{rcp}.
For the Gaussian reference, the unique four-level MMSE solution consists of two mirrored pairs (\cref{prop:lossless}).
IAC uses this structure and fits two positive magnitudes $0<m_{1,r}<m_{2,r}$ to each empirical row:
\begin{align}
\{m_{1,r}^\star,m_{2,r}^\star\}
&\;=\; \arg\min_{0<m_1<m_2}
\frac{1}{d_i}\sum_{j=1}^{d_i}
\min_{c\,\in\,\{\pm m_1,\pm m_2\}}
(w'_{rj}-c)^2 .
\label{eq:mmse}
\end{align}
We optimize this objective by alternating nearest-level assignment with magnitude updates.
For fixed assignments, each magnitude is the mean absolute value of the weights assigned to it~\citep{lloyd1982,max1960}.
Stage 1 regularizes the marginals, and Stage 2 fits four reconstruction values to minimize their distortion, preserving information within the two-bit budget.

\Needspace{4\baselineskip}
\textbf{Stage 3: Exact two-sign-plane decoding.}

To preserve the fitted reconstruction values in efficient arithmetic computation, we express the nonuniform levels as sums of two signed amplitudes.
For each row, setting $a_{1,r}=(m_{2,r}+m_{1,r})/2$ and $a_{2,r}=(m_{2,r}-m_{1,r})/2$ gives
\begin{equation}
\hat{\boldsymbol{w}}'
=
\operatorname{diag}(\boldsymbol{a}_1)\boldsymbol{s}_1+\operatorname{diag}(\boldsymbol{a}_2)\boldsymbol{s}_2,
\label{eq:decomp}
\end{equation}
where $\boldsymbol{a}_1,\boldsymbol{a}_2\in\mathbb{R}^{d_o}$ are row-wise amplitudes and $\boldsymbol{s}_1,\boldsymbol{s}_2\in\{-1,+1\}^{d_o\times d_i}$ are sign planes.
The four sign combinations reproduce $\{\pm m_{1,r},\pm m_{2,r}\}$ exactly, so this conversion adds no reconstruction error in exact arithmetic.
Since the amplitudes are shared within each row, they factor outside the dot products:
\begin{equation}
\boldsymbol{x}'(\hat{\boldsymbol{w}}')^\top
=
(\boldsymbol{x}'\boldsymbol{s}_1^\top)\operatorname{diag}(\boldsymbol{a}_1)
+(\boldsymbol{x}'\boldsymbol{s}_2^\top)\operatorname{diag}(\boldsymbol{a}_2).
\label{eq:gemm-decode}
\end{equation}
\Needspace{4.3in}
\begin{wrapfigure}{r}{0.68\textwidth}
\vspace{-8pt}
\centering
\begin{minipage}[c]{2.3in}
\includegraphics[width=\linewidth,trim=0 3bp 0 0,clip]{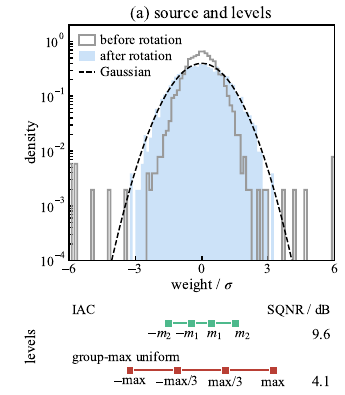}
\end{minipage}\hfill
\begin{minipage}[c]{1.38in}
\includegraphics[width=\linewidth,trim=0 2bp 0 0,clip]{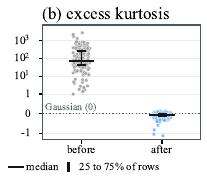}\\[3pt]
\includegraphics[width=\linewidth,trim=0 1.5bp 0 0,clip]{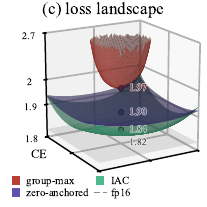}
\end{minipage}
\caption{IAC source shaping and adaptive level placement on LLaMA-2-7B at W2. (a) Rotation turns a heavy-tailed row into an approximately symmetric marginal, and the fitted levels improve SQNR over the group-max grid. (b) The reduction in excess kurtosis holds across input-rotated projections. (c) With one block quantized, the fitted levels lie in a low-cross-entropy region of the model-level loss surface. The upper loss range is compressed for visibility.}
\label{fig:iac}
\vspace{-6pt}
\end{wrapfigure}

The linear operation thus uses two sign-plane reductions and output scaling, avoiding per-element codebook lookup.
Both paths can be fused without materializing a decoded weight matrix.
With quantized activations, the dot products can use integer arithmetic.
Each weight is represented by two sign bits, with two fp16 amplitudes shared across each row.
The complete quantization procedure is given in \cref{app:algorithm}.

By spreading outlier energy across coordinates, Stage 1 reduces the influence of rare extremes on the scalar source and brings its marginals closer to the maximum-entropy Gaussian reference (\cref{fig:iac}a,b).
The shaped density remains nonuniform, so a uniform grid still allocates reconstruction precision poorly across its probability mass.
Stage 2 uses this regularized source to fit four levels where they most reduce expected squared error, converting source shaping into a better use of the two-bit reconstruction budget.
On the illustrated rotated row, the fitted levels raise SQNR from 4.1\,dB with the group-max grid to 9.6\,dB (\cref{fig:iac}a).
The covariance between sequential sensitivity and squared error links scalar distortion to layer loss (\cref{app:theory}); the more even sensitivities observed after rotation help align these objectives.
The fitted levels also lie in a low-cross-entropy region of the model-level landscape (\cref{fig:iac}c).
Source shaping thus supplies a more favorable basis for scalar reconstruction, and minimum-distortion level placement exploits it to retain as much information as possible within the two-bit budget, while Stage 3 maintains efficient arithmetic execution.

IAC therefore combines source-adaptive reconstruction with efficient two-bit scalar execution.
It reduces fixed-grid mismatch and captures source-adaptation benefits of vector quantization, while retaining a regular arithmetic decoder without codebook lookup or decoded-weight materialization.

\subsection{Knowledge-Intensive Calibration for Consistent Quantization}
\label{sec:kic}

Although IAC minimizes weight distortion within the scalar representation, residual errors can propagate through nonlinear layers and alter predictions on a small subset of inputs.
Knowledge-Intensive Calibration (KIC) allocates a limited calibration budget to these inputs.
It ranks complete sequences by residual teacher-student disagreement, then refines only the quantization parameters on the selected sequences.
Let $f$ denote the full-precision teacher and $\hat f$ the IAC-quantized model.
At position $t$, their predictive distributions are restricted and renormalized on the teacher's top-$k_v$ support to obtain $\widetilde p_f$ and $\widetilde p_{\hat f}$.
KIC measures the damage at each position by the forward divergence
\begin{equation}
d_t(x)
\;=\;
\mathrm{KL}\!\left(
\widetilde p_f(\cdot\mid x_{<t})
\,\big\|\,
\widetilde p_{\hat f}(\cdot\mid x_{<t})
\right).
\label{eq:damage}
\end{equation}
For a sequence $x$, let $\mathcal{T}_\rho(x)$ index token divergences at or above their $(1-\rho)$ quantile.
The sequence score is their mean
\begin{equation}
s_\rho(x)
\;=\;
\frac{1}{|\mathcal{T}_\rho(x)|}
\sum_{t\in\mathcal{T}_\rho(x)} d_t(x),
\qquad \rho=0.25.
\label{eq:damage-score}
\end{equation}
KIC ranks the candidate pool by $s_\rho(x)$ and retains the top $k$ complete sequences as $\mathcal{S}_k$.

On the selected sequences, KIC adjusts the reconstruction levels to reduce prediction error across the model.
It minimizes the objective
\begin{equation}
\ell_{\mathrm{KIC}}
\;=\;
\mathbb{E}_{(x,y)\in\mathcal{S}_k}
\!\left[
\ell_{\mathrm{CE}}(y,\boldsymbol{p}_s)
+ \lambda\,\ell_{\mathrm{CE}}(\boldsymbol{p}_t,\boldsymbol{p}_s)
\right],
\qquad \lambda=0.1.
\label{eq:repair}
\end{equation}
Here $\boldsymbol{p}_t$ renormalizes the teacher distribution on its cached top-512 tokens, with zero mass elsewhere; $\boldsymbol{p}_s$ is the student distribution.
The hard-target term retains corpus supervision, while minimizing the soft-target cross-entropy is equivalent to minimizing teacher-student KL divergence~\citep{kd}.
Only the two reconstruction magnitudes $(m_{1,r},m_{2,r})$ per row are trainable; weight signs and bin assignments remain fixed.
The decoder amplitudes $(a_{1,r},a_{2,r})$ follow from the mapping in Stage 3.
These updates stay within the two-sign-plane family established by IAC, preserving the stored two-bit representation and its arithmetic decoder.

\begin{figure}[t]
\centering
\includegraphics[width=\textwidth,trim=0 3bp 0 0,clip]{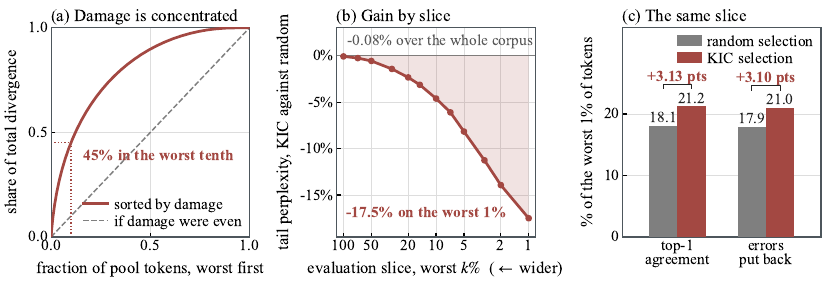}
\caption{The damage KIC targets is concentrated, and KIC improves the tokens that carry it. On LLaMA-2-7B at W2A4KV4, (a) the worst decile of tokens carries 45\% of the total divergence. On Qwen3-8B-Base at W2A16, against a random draw of the same size on a slice ranked before either arm is trained, (b) the tail perplexity gain grows from 0.08\% over the corpus to 17.5\% on the worst 1\%, and (c) on that slice, top-1 agreement and repaired errors rise by 3.13 and 3.10 points, where tail perplexity is 79.8 against 96.7.}
\label{fig:kic}
\end{figure}
Teacher-student divergence measures the prediction change caused by quantization, which local reconstruction MSE cannot locate after errors propagate through the model.
Teacher entropy instead measures uncertainty in the original prediction and can be high even when quantization leaves that prediction unchanged.
\Cref{fig:kic} shows what selecting on this score buys: on the worst 1\% of held-out tokens by pre-calibration divergence, tail perplexity falls from 96.7 to 79.8 against a random draw of the same size, while the two arms stay within 0.08\% of each other on the corpus as a whole.  % [T:figs/F4/kic_tail.csv]
Because a sequence may contain only a few strongly altered positions, the tail mean prevents their damage from being diluted by positions near agreement and reduces dependence on a single extreme value.
This upper-tail average follows the conditional value-at-risk criterion~\citep{cvar}.
Selecting the top $k$ sequences, therefore, maximizes the sum of sequence scores under the calibration budget.

KIC therefore concentrates a small calibration budget on the residual errors that most alter model predictions.
Scoring requires one offline forward pass of each model over the candidate pool and adds no inference-time cost.
On one shared LLaMA-2-7B pool, rankings select sets with 0.80 to 0.90 overlap across A16KV16, A4KV8, and A4KV4, compared with 0.125 for random selection.  % [T:figs/data/kic_overlap.npz]
This overlap supports reusing a ranking across precision configurations of the same base model when rescoring is unnecessary.
The small trainable parameter set also limits memory for gradients and optimizer states.

% ====================================================================
\section{Experiments}
\label{sec:experiments}

\subsection{Setup}
\label{sec:setup}

\textbf{Quantization.}
We evaluate weight-only quantization at W2A16KV16 and joint weight, activation, and KV-cache quantization at W2A4KV4, where W-A-KV denotes the three bit widths.
\method{} quantizes weights to two bits with per-row reconstruction levels, while embeddings and the output head remain in FP16.
The downstream checkpoints use 4096 KIC-selected sequences, except LLaMA-2-13B, which uses 1024.
Each \method{} row comes from one calibration run on a mixed-domain pool.
Calibration settings and implementation details are provided in Appendix~\ref{app:protocol}.

\textbf{Models and baselines.}
We evaluate LLaMA-2-7B/13B~\citep{llama2}, LLaMA-3-8B~\citep{llama3}, Qwen3-8B-Base~\citep{qwen3}, and LLaMA-30B~\citep{llama}.
LLaMA-3-8B results are provided in Appendix~\ref{app:llama3-results}.
Our primary controlled baseline, Base, uses randomized rotation, error compensation, and a uniform grid without finetuning.
At matched W-A-KV configurations, we also compare with QuaRot~\citep{quarot}, ResComp~\citep{rescomp}, and DB-LLM~\citep{dbllm} as scalar baselines, and AQLM~\citep{aqlm} as a codebook accuracy reference.
Additional comparisons are in Appendix~\ref{app:results}.

\textbf{Evaluation.}
We measure perplexity on WikiText2~\citep{wikitext2} and the segment-sampled C4~\citep{c4} validation split with a context length of 2048.
Downstream evaluation includes zero-shot accuracy on PiQA~\citep{piqa}, ARC-e and ARC-c~\citep{arc}, HellaSwag~\citep{hellaswag}, WinoGrande~\citep{winogrande}, BoolQ~\citep{boolq}, and OpenBookQA~\citep{obqa}, together with LAMBADA~\citep{lambada} and 5-shot MMLU~\citep{mmlu}.
We report the seven-task average separately from LAMBADA and MMLU and use the LM Evaluation Harness~\citep{evalharness} for task evaluation.

\subsection{Quantization accuracy}
\label{sec:main-results}

\begin{table}[!t]
\caption{\method{} achieves the lowest WikiText2 perplexity among the compared scalar methods across model families, under weight-only and joint weight, activation, and KV-cache quantization. The 7-task average covers ARC-Easy, ARC-Challenge, HellaSwag, PIQA, WinoGrande, BoolQ, and OpenBookQA, and the 5-task average the first five of these; complete task results are in Appendix~\ref{app:results}.}
\label{tab:main}
\begin{center}
\scriptsize
\renewcommand{\arraystretch}{0.96}
\setlength{\tabcolsep}{2.2pt}
\resizebox{\textwidth}{!}{%
\begin{tabular}{ll|cc|cccc|cc|ccc}
\toprule
\multirow{2}{*}{\textbf{Method}} & \multirow{2}{*}{\textbf{W-A-KV}}
 & \multicolumn{2}{c|}{\textbf{Perplexity}\,$\downarrow$}
 & \multirow{2}{*}{\textbf{HellaS}} & \multirow{2}{*}{\textbf{PIQA}} & \multirow{2}{*}{\textbf{WinoG}} & \multirow{2}{*}{\textbf{OBQA}}
 & \multicolumn{2}{c|}{\textbf{Avg.}\,$\uparrow$}
 & \multicolumn{2}{c}{\textbf{LAMBADA}} & \multirow{2}{*}{\textbf{MMLU}$_{5}$} \\
\cmidrule(lr){3-4}\cmidrule(lr){9-10}\cmidrule(lr){11-12}
 &  & Wiki2 & C4 &  &  &  &  & 7-task & 5-task & Std & OpenAI & \\
\midrule
\multicolumn{13}{l}{\textbf{LLaMA-2-7B}} \\
FP16 & 16-16-16 & 5.47 & 6.97 & 76.16 & 78.78 & 69.46 & 44.0 & 66.67 & 68.66 & 68.81 & 73.65 & 45.94 \\
AQLM$^{\ast}$ & 2-16-16 & 7.24 & \textbf{8.96} & 66.74 & 74.10 & 62.51 & 36.6 & 58.04 & 60.08 & 56.26 & 67.34 & 33.03 \\
\hdashline[1.5pt/1.5pt]
%RTN & 2$^{\ddagger}$-16-16 & 431.97 &  &  &  &  &  &  &  &  &  & \\
QuaRot & 2-16-16 & 24.52 & 39.39 & 33.30 & 58.98 & 52.17 & 28.2 & 42.45 & 41.50 & 15.72 & 17.04 & 25.87 \\
DB-LLM & 2-16-16 & 8.89 & 19.85 & 51.99 & 62.13 & 55.96 & 32.6 & 47.80 & 48.16 & 32.19 & 45.47 & 24.95 \\
%EfficientQAT\,(QAT) & 2$^{\dagger}$-16-16 & 7.16 & \textbf{8.90} & \textbf{67.43} & \textbf{75.03} & \textbf{65.51} & \textbf{39.0} & \textbf{60.26} & \textbf{62.16} & \textbf{59.19} & 67.32 & \textbf{33.55} \\
%EfficientQAT\,(QAT) & 2$^{\ddagger}$-16-16 & 6.86 & 8.51 & 68.88 & 74.86 & 64.64 & 39.0 & 60.80 & 62.10 & 57.91 & 66.68 & 33.04 \\
ResComp & 2-16-16 & 10.28 & 19.68 & 44.15 & 64.85 & 58.25 & 31.0 & 47.70 & 47.98 & 24.92 & 33.07 & 27.17 \\
Base (uniform grid) & 2-16-16 & 9.56 & 18.43 & 44.50 & 66.16 & 57.93 & 29.2 & 48.00 & 48.91 & 24.45 & 32.37 & 26.70 \\
\textbf{\method{}} & 2-16-16 & \textbf{6.62} & 9.53 & \textbf{68.33} & \textbf{74.97} & \textbf{66.14} & \textbf{40.4} & \textbf{59.63} & \textbf{61.78} & \textbf{61.89} & \textbf{68.87} & \textbf{37.29} \\
\hdashline[1.5pt/1.5pt]
% 2-4-8 held back: no rerun under the current pipeline, so it would be the only row of the
% previous generation left in this block.
%\textbf{\method{}} & 2-4-8 & 6.54 & 9.51 & 63.32 & 71.49 & 61.64 & 37.6 & 56.28 & 57.73 & 48.61 & 57.69 & 26.10 \\
%\hdashline[1.5pt/1.5pt]
QuaRot & 2-4-4 & 47.23 & 77.49 & 30.02 & 53.54 & 50.43 & 25.8 & 39.06 & 37.84 & 6.04 & 8.00 & 24.92 \\
ResComp & 2$^{\dagger}$-4-4 & 11.06 & 20.85 & 41.14 & 61.15 & 54.38 & 31.0 & 46.04 & 45.73 & 18.40 & 26.26 & 24.90 \\
ResComp & 2-4-4 & 12.78 & 23.91 & 40.19 & 61.81 & 52.41 & 28.0 & 44.81 & 44.79 & 18.77 & 25.46 & 26.01 \\
Base (uniform grid) & 2-4-4 & 11.27 & 21.41 & 39.59 & 62.51 & 54.30 & 28.4 & 45.01 & 44.94 & 15.18 & 23.35 & 25.55 \\
\textbf{\method{}} & 2-4-4 & \textbf{6.96} & \textbf{9.97} & \textbf{65.26} & \textbf{73.78} & \textbf{62.04} & \textbf{36.0} & \textbf{57.25} & \textbf{59.98} & \textbf{57.44} & \textbf{65.01} & \textbf{33.99} \\
\midrule
\multicolumn{13}{l}{\textbf{LLaMA-2-13B}} \\
FP16 & 16-16-16 & 4.88 & 6.47 & 79.62 & 80.25 & 72.45 & 45.2 & 69.27 & 71.52 & 70.68 & 76.54 & 55.43 \\
AQLM$^{\ast}$ & 2-16-16 & 6.06 & \textbf{7.80} & 72.68 & \textbf{77.58} & 64.80 & 39.8 & 63.25 & 65.42 & \textbf{63.30} & \textbf{72.00} & 45.38 \\
\hdashline[1.5pt/1.5pt]
QuaRot & 2-16-16 & 12.17 & 24.11 & 45.14 & 62.30 & 56.59 & 28.8 & 46.32 & 46.59 & 31.11 & 40.42 & 25.44 \\
DB-LLM & 2-16-16 & 7.97 & 17.61 & 59.40 & 68.77 & 59.67 & 34.4 & 53.75 & 55.28 & 41.63 & 51.14 & 26.62 \\
ResComp & 2-16-16 & 7.90 & 14.05 & 53.24 & 68.66 & 62.83 & 33.4 & 52.46 & 54.12 & 38.95 & 50.01 & 27.97 \\
Base (uniform grid) & 2-16-16 & 7.78 & 14.43 & 50.14 & 67.63 & 62.75 & 33.2 & 51.63 & 53.14 & 34.80 & 48.40 & 27.38 \\
\textbf{\method{}} & 2-16-16 & \textbf{5.86} & 8.43 & \textbf{74.39} & 77.53 & \textbf{69.14} & \textbf{40.8} & \textbf{63.97} & \textbf{66.34} & 62.90 & 71.43 & \textbf{45.54} \\
\hdashline[1.5pt/1.5pt]
QuaRot & 2-4-4 & \multicolumn{2}{c|}{\emph{diverged}} & 25.96 & 50.00 & 51.78 & 27.2 & 35.29 & 36.40 & \multicolumn{3}{c}{\emph{at chance}} \\
% LLaMA-2-13B ResComp 2-dagger-4-4 withheld.
%ResComp & 2$^{\dagger}$-4-4 & 8.41 & 14.69 & 51.89 & 67.08 & 60.14 & 32.8 & 50.77 & 52.06 &  &  & \\
ResComp & 2-4-4 & 9.70 & 19.10 & 43.38 & 64.53 & 59.51 & 31.4 & 47.88 & 48.32 & 26.41 & 40.60 & 27.19 \\
Base (uniform grid) & 2-4-4 & 8.91 & 16.54 & 45.55 & 65.61 & 57.85 & 29.8 & 48.29 & 49.22 & 24.80 & 37.96 & 25.96 \\
\textbf{\method{}} & 2-4-4 & \textbf{6.04} & \textbf{8.81} & \textbf{72.64} & \textbf{76.28} & \textbf{65.43} & \textbf{39.4} & \textbf{62.14} & \textbf{64.43} & \textbf{61.44} & \textbf{69.88} & \textbf{43.74} \\
\midrule
\multicolumn{13}{l}{\textbf{Qwen3-8B-Base}} \\
FP16 & 16-16-16 & 7.00 & 10.42 & 78.60 & 79.54 & 72.06 & 41.8 & 70.18 & 73.26 & 66.02 & 70.83 & 76.84 \\
\hdashline[1.5pt/1.5pt]
QuaRot & 2-16-16 & 12.58 & 23.02 & 52.47 & 67.79 & 61.88 & 31.6 & 53.29 & 55.55 & 43.78 & 49.56 & 43.93 \\
DB-LLM & 2-16-16 & 11.62 & 23.33 & 50.92 & 64.64 & 57.46 & 35.6 & 49.84 & 50.88 & 32.18 & 40.68 & 27.55 \\
ResComp & 2-16-16 & 11.21 & 22.82 & 49.11 & 65.51 & 61.25 & 31.4 & 51.56 & 52.55 & 39.90 & 43.84 & 36.06 \\
Base (uniform grid) & 2-16-16 & 11.55 & 25.78 & 45.20 & 64.42 & 58.48 & 30.8 & 49.43 & 50.52 & 40.52 & 40.33 & 30.69 \\
\textbf{\method{}} & 2-16-16 & \textbf{8.84} & \textbf{13.64} & \textbf{70.28} & \textbf{74.97} & \textbf{67.96} & \textbf{42.0} & \textbf{63.85} & \textbf{66.00} & \textbf{57.91} & \textbf{65.11} & \textbf{59.94} \\
\hdashline[1.5pt/1.5pt]
QuaRot & 2-4-4 & 20.38 & 40.50 & 43.89 & 60.66 & 53.12 & 29.0 & 46.60 & 47.03 & 30.06 & 32.49 & 30.52 \\
ResComp & 2$^{\dagger}$-4-4 & 13.47 & 25.73 & 46.60 & 64.09 & 56.91 & 32.6 & 50.00 & 50.68 & 31.30 & 36.77 & 32.43 \\
ResComp & 2-4-4 & 14.85 & 31.40 & 43.53 & 63.33 & 54.22 & 28.0 & 47.65 & 48.62 & 30.49 & 35.69 & 30.86 \\
Base (uniform grid) & 2-4-4 & 15.56 & 36.47 & 39.39 & 60.45 & 52.72 & 29.4 & 45.06 & 44.69 & 22.10 & 25.60 & 28.30 \\
\textbf{\method{}} & 2-4-4 & \textbf{10.69} & \textbf{16.59} & \textbf{62.86} & \textbf{72.09} & \textbf{60.22} & \textbf{39.0} & \textbf{59.17} & \textbf{61.24} & \textbf{49.62} & \textbf{54.76} & \textbf{47.60} \\
\midrule
\multicolumn{13}{l}{\textbf{LLaMA-30B}} \\
FP16 & 16-16-16 & 4.10 & 5.98 & 83.01 & 81.83 & 75.85 & 49.2 & 72.36 & 74.70 & 73.06 & 77.61 & 58.30 \\
\hdashline[1.5pt/1.5pt]
% TODO: QuaRot 2-16-16 for this block.
ResComp & 2-16-16 & 6.45 & 10.38 & 63.98 & 74.32 & 69.30 & 38.0 & 60.06 & 62.73 & 54.57 & 67.20 & 32.00 \\
Base (uniform grid) & 2-16-16 & 6.33 & 10.10 & 64.91 & 74.70 & 68.59 & 39.0 & 59.91 & 62.76 & 54.63 & 68.33 & 31.29 \\
\textbf{\method{}} & 2-16-16 & \textbf{5.14} & \textbf{7.64} & \textbf{79.23} & \textbf{78.67} & \textbf{71.98} & \textbf{42.4} & \textbf{67.06} & \textbf{69.51} & \textbf{68.21} & \textbf{74.83} & \textbf{49.96} \\
\hdashline[1.5pt/1.5pt]
QuaRot & 2-4-4 & 99.92 & 961.80 & 46.67 & 61.04 & 57.38 & 31.6 & 48.31 & 49.10 & 21.42 & 22.74 & 25.99 \\
ResComp & 2-4-4 & 6.88 & 11.14 & 61.29 & 72.85 & 67.48 & 36.0 & 56.62 & 59.47 & 48.44 & 62.22 & 29.67 \\
Base (uniform grid) & 2-4-4 & 6.86 & 11.04 & 60.44 & 71.44 & 66.77 & 37.2 & 56.60 & 59.08 & 48.38 & 61.91 & 28.42 \\
\textbf{\method{}} & 2-4-4 & \textbf{5.71} & \textbf{7.79} & \textbf{76.29} & \textbf{77.31} & \textbf{71.19} & \textbf{43.2} & \textbf{65.31} & \textbf{68.18} & \textbf{63.01} & \textbf{71.92} & \textbf{38.64} \\
\bottomrule
\end{tabular}}
\par
$^{\ast}$ Lookup decoding. $^{\dagger}$ Group size 128.
\end{center}
% \vspace{-6mm}
\end{table}

\textbf{Weight-only language-modeling accuracy.}
\Cref{tab:main} evaluates the complete method at W2A16KV16, where \method{} achieves the lowest WikiText2 perplexity among the compared scalar methods across model families.
On LLaMA-2-7B, perplexity decreases from 9.56 with Base to 6.62, compared with 8.89 for DB-LLM and 7.24 for AQLM.
This advantage extends to LLaMA-2-13B, Qwen3-8B-Base, and LLaMA-30B; LLaMA-3-8B results are in Appendix~\ref{app:llama3-results}.

\textbf{Downstream accuracy.}
The improvements extend to downstream tasks.
At W2A16KV16, \method{} raises the seven-task average on LLaMA-2-7B from 48.00\% with Base to 59.63\% and improves over the compared scalar baselines across model families.
LAMBADA shows a similar trend, while AQLM remains stronger on some LLaMA-2-13B and LLaMA-3-8B tasks.
These results show that the gains extend beyond language-modeling perplexity.

\textbf{Joint weight-activation-KV quantization.}
Under W2A4KV4, \method{} retains the lowest WikiText2 perplexity and the highest seven-task average among the compared scalar methods in each model block.
On LLaMA-2-7B, perplexity increases from 6.62 to 6.96, 1.49 above FP16.
The complete method thus remains effective when activation and cache quantization add error.

\subsection{Component analysis}
\label{sec:ablations}

\begin{table}[!htbp]
\caption{Component ablation on LLaMA-2-7B. \emph{Base} uses rotation, GPTAQ, and a uniform grid without finetuning. Rows add IAC, KIC, or both; Base + IAC + KIC is \method{}. All rows use 512 WikiText2 calibration sequences. Bold marks the best quantized result per precision configuration.}
\label{tab:component-ablation}
\begin{center}
\scriptsize
\renewcommand{\arraystretch}{1.08}
\setlength{\tabcolsep}{2.2pt}
\resizebox{\textwidth}{!}{%
\begin{tabular}{ll|cc|ccccccc|cc|ccc}
\toprule
\multirow{2}{*}{\textbf{Variant}} & \multirow{2}{*}{\textbf{W-A-KV}}
 & \multicolumn{2}{c|}{\textbf{Perplexity}\,$\downarrow$}
 & \multicolumn{2}{c}{\textbf{ARC}} & \multirow{2}{*}{\textbf{HellaS}} & \multirow{2}{*}{\textbf{PIQA}} & \multirow{2}{*}{\textbf{WinoG}} & \multirow{2}{*}{\textbf{BoolQ}} & \multirow{2}{*}{\textbf{OBQA}}
 & \multicolumn{2}{c|}{\textbf{Avg.}\,$\uparrow$}
 & \multicolumn{2}{c}{\textbf{LAMBADA}} & \multirow{2}{*}{\textbf{MMLU}$_{5}$} \\
\cmidrule(lr){3-4}\cmidrule(lr){5-6}\cmidrule(lr){12-13}\cmidrule(lr){14-15}
 &  & Wiki2 & C4 & Easy & Chg &  &  &  &  &  & 7-task & 5-task & Std & OpenAI &  \\
\midrule
FP16 & 16-16-16 & 5.47 & 6.97 & 73.86 & 45.05 & 76.16 & 78.78 & 69.46 & 79.36 & 44.0 & 66.67 & 68.66 & 68.81 & 73.65 & 45.94 \\
\midrule
Base & 2-16-16 & 9.56 & 18.43 & 49.24 & 26.71 & 44.50 & 66.16 & 57.93 & 62.29 & 29.2 & 48.00 & 48.91 & 24.45 & 32.37 & 26.70 \\
Base + IAC & 2-16-16 & 7.78 & 12.53 & 51.43 & 28.84 & 52.98 & 67.03 & 61.33 & 62.35 & 34.2 & 51.17 & 52.32 & 39.14 & 51.64 & 28.73 \\
Base + KIC & 2-16-16 & 9.56 & 18.42 & 49.37 & 26.79 & 44.45 & 66.16 & 57.46 & 62.29 & 29.4 & 47.99 & 48.85 & 24.49 & 32.52 & 26.73 \\
Base + IAC + KIC & 2-16-16 & \textbf{7.11} & \textbf{9.74} & \textbf{60.19} & \textbf{36.35} & \textbf{65.70} & \textbf{74.05} & \textbf{63.93} & \textbf{68.69} & \textbf{37.4} & \textbf{58.04} & \textbf{60.04} & \textbf{56.90} & \textbf{67.16} & \textbf{29.54} \\
% The full-pipeline row carried the same arm as IAC + KIC above, differing only in the
% calibration recipe; with both on the current pipeline it is pure duplication.
%\method{} (full pipeline) & 2-16-16 & 7.11 & \textbf{9.74} & \textbf{60.19} & \textbf{36.35} & \textbf{65.70} & \textbf{74.05} & \textbf{63.93} & 68.69 & 37.4 & \textbf{58.04} & \textbf{60.04} & \textbf{56.90} & \textbf{67.16} & \textbf{29.54} \\
\addlinespace[3pt]
Base & 2-4-4 & 11.27 & 21.41 & 43.22 & 25.09 & 39.59 & 62.51 & 54.30 & 61.99 & 28.4 & 45.01 & 44.94 & 15.18 & 23.35 & 25.55 \\
Base + IAC & 2-4-4 & 9.04 & 15.56 & 51.85 & 30.20 & 46.96 & 66.38 & 58.56 & 62.45 & 32.4 & 49.83 & 50.79 & 26.82 & 37.69 & \textbf{26.46} \\
Base + KIC & 2-4-4 & 11.33 & 21.50 & 44.11 & 25.26 & 40.46 & 62.79 & 55.41 & 62.02 & 28.4 & 45.49 & 45.61 & 15.85 & 23.33 & 25.69 \\
Base + IAC + KIC & 2-4-4 & \textbf{7.53} & \textbf{10.20} & \textbf{54.46} & \textbf{34.47} & \textbf{63.90} & \textbf{71.49} & \textbf{62.27} & \textbf{63.98} & \textbf{34.4} & \textbf{55.00} & \textbf{57.32} & \textbf{50.28} & \textbf{61.79} & 26.09 \\
\bottomrule
\end{tabular}}
\end{center}
\end{table}

\Cref{tab:component-ablation} compares the components under a shared 512-sequence calibration protocol.
At W2A16KV16, adding IAC to Base lowers WikiText2 perplexity from 9.56 to 7.78; applying KIC afterward further lowers it to 7.11 and raises the seven-task average from 51.17\% to 58.04\%.
The same progression holds at W2A4KV4.
KIC alone leaves Base perplexity almost unchanged in both settings, whereas it improves the IAC initialization.
These results support combining adaptive reconstruction with subsequent predictive calibration.
Additional level-placement and representation controls are provided in \cref{tab:quantizer-ablation}.
To isolate the contribution of sequence selection, we fix the candidate pool, calibration budget, and optimization recipe while varying only the selected sequences.
On Qwen3-8B, tail-damage selection lowers WikiText2 perplexity from 17.05 with random selection to 16.69 (Appendix~\ref{app:ablations}, \cref{tab:selection}).

\subsection{Inference efficiency}
\label{sec:efficiency}

For LLaMA-2-7B, \method{} reduces linear-weight storage by 8.0$\times$ relative to FP16.
At batch size 1 on an RTX A6000, its W2A16 linear stack decodes in 5.11\,ms/token, against 8.83\,ms for QuIP\# E8P and 9.93\,ms for AQLM $2\times8$ at matched two-bit storage, and 4.72\,ms for the uniform grid whose perplexity is 9.56 rather than 6.62 (\cref{fig:tradeoff,tab:latency}).
Full-model decode throughput in HuggingFace eager mode is 47.0 tokens/s, versus 35.9 for FP16 (\cref{tab:e2e}).
Appendix~\ref{app:kernel} gives per-method timings and measurement details.

% ====================================================================

\section{Discussion}
\label{sec:discussion}

At a fixed W2A4KV4 budget, further gains in predictive fidelity may come from jointly optimizing source shaping, reconstruction levels, and discrete assignments.
Quantization loss could guide the source transform, while prediction feedback could refine both levels and assignments.
Accounting for interactions among weight, activation, and cache errors as they propagate through layers and attention offers another direction for reducing loss within the same compact representation.

Realizing the deployment benefits of W2A4KV4 also requires kernels adapted to its mixed precisions.
The evaluated GPU lacks native INT2-by-INT4 matrix multiplication; our current fused kernel accelerates two-bit weights with 16-bit activations at batch size one.
Integrating sign-plane computation with four-bit activations and cache access, and extending fusion to prefill and larger batches, are directions for translating joint quantization into broader inference gains.

% ====================================================================
\section{Conclusion}
\label{sec:conclusion}

\method{} preserves predictive fidelity at two bits by combining source-adaptive scalar reconstruction with targeted calibration.
IAC first shapes the weight distribution, then places reconstruction levels to minimize distortion, preserving the resulting level set through exact arithmetic decoding.
KIC uses residual teacher-student disagreement to select calibration sequences and refines only the two reconstruction magnitudes per row.
Experiments on LLaMA and Qwen show improved accuracy over matched scalar baselines in weight-only and joint quantization.
The fused implementation also realizes storage savings and faster decoding, demonstrating that source-adaptive reconstruction can retain the execution advantages of scalar quantization.

% ====================================================================
\clearpage
\clearpage

\bibliography{refs}
\bibliographystyle{iclr2027_conference}

\clearpage
\appendix

\section{Theoretical analysis}
\label{app:theory}

Throughout, $\boldsymbol{w}\in\mathbb{R}^{m\times n}$ is a weight matrix with rows $\boldsymbol{w}_r$, $\sigma_r^2=\frac1n\|\boldsymbol{w}_r\|_2^2$, and $\boldsymbol{q}$ is a randomized Hadamard transform built from $\boldsymbol{u}_n=n^{-1/2}\boldsymbol{u}$, with $\boldsymbol{u}$ a Sylvester-Hadamard matrix ($\boldsymbol{u}^\top \boldsymbol{u}=n\boldsymbol{i}$, first row all ones, $\boldsymbol{i}$ the identity) and $\boldsymbol{r}=\mathrm{diag}(\boldsymbol{\xi})$, $\boldsymbol{\xi}\in\{\pm1\}^n$ uniform; the two orderings $\boldsymbol{q}=\boldsymbol{r}\boldsymbol{u}_n$ and $\boldsymbol{q}=\boldsymbol{u}_n\boldsymbol{r}$ are distinguished where it matters.
Rotated rows are $\boldsymbol{w}'_r=\boldsymbol{w}_r\boldsymbol{q}$, and $\overline{\boldsymbol{w}'_r}:=\frac1n\sum_j w'_{rj}$ is the row mean.

\subsection{Source centering}

\begin{lemma}[Mean under randomized rotation]
\label{lem:zeromean}
$\boldsymbol{u}_n\boldsymbol{1}=\sqrt{n}\,\boldsymbol{e}_1$.
With $\boldsymbol{q}=\boldsymbol{u}_n\boldsymbol{r}$ every rotated row has mean zero in expectation over $\boldsymbol{\xi}$, with $\mathbb{E}_{\boldsymbol{\xi}}[\overline{\boldsymbol{w}'_r}]=0$ and $\mathrm{sd}_{\boldsymbol{\xi}}[\overline{\boldsymbol{w}'_r}]=\sigma_r n^{-1/2}$ regardless of outliers.
With $\boldsymbol{q}=\boldsymbol{r}\boldsymbol{u}_n$ the row mean is $n^{-1/2}\xi_1w_{r1}$: zero in expectation, with standard deviation $n^{-1/2}|w_{r1}|$.
\end{lemma}
\begin{proof}
$\boldsymbol{u}_n\boldsymbol{1}=n^{-1/2}\boldsymbol{u}\boldsymbol{1}=n^{-1/2}(n,0,\dots,0)^\top=\sqrt{n}\,\boldsymbol{e}_1$, since the first row of $\boldsymbol{u}$ is all ones and the other rows are orthogonal to it.
For $\boldsymbol{q}=\boldsymbol{r}\boldsymbol{u}_n$ and a fixed $\boldsymbol{r}$, $w'_{rj}=n^{-1/2}\sum_i w_{ri}\xi_iu_{ij}$, so $\overline{\boldsymbol{w}'_r}=n^{-3/2}\sum_i w_{ri}\xi_i\sum_j u_{ij}=n^{-1/2}w_{r1}\xi_1$ by $\sum_j u_{ij}=n\,\mathbf{1}[i=1]$; this is not small when $w_{r1}$ is an outlier, and randomizing $\xi_1$ gives expectation $0$ and standard deviation $n^{-1/2}|w_{r1}|$.
For $\boldsymbol{q}=\boldsymbol{u}_n\boldsymbol{r}$, $\overline{\boldsymbol{w}'_r}=\frac1n\sum_i w_{ri}(\boldsymbol{u}_n\boldsymbol{\xi})_i$ with $\mathbb{E}[(\boldsymbol{u}_n\boldsymbol{\xi})_i(\boldsymbol{u}_n\boldsymbol{\xi})_k]=\delta_{ik}$, so the expectation is $0$ and the standard deviation is $\frac1n\|\boldsymbol{w}_r\|_2=\sigma_r n^{-1/2}$.
\end{proof}
Both orderings center the row mean in expectation, while its variance depends on the placement of the signs.
For $\boldsymbol{q}=\boldsymbol{r}\boldsymbol{u}_n$, each output coordinate is a Rademacher sum with variance $\sigma_r^2$ and the same marginal law.
A Gaussian approximation applies when no individual summand dominates the row energy; the empirical effect of the rotation is examined in \cref{fig:iac}.
The source model and the analysis below do not require pairwise independence of the rotated coordinates.

\subsection{Scalar distortion and layer loss}

\begin{proposition}[Sequential sensitivity decomposition]
\label{prop:unweighted}
Let $\boldsymbol{h}=2\boldsymbol{x}\boldsymbol{x}^\top+\lambda\boldsymbol{i}$ be positive definite, where rotated calibration activations are arranged as columns and $\lambda\ge0$ is the damping parameter. Define $\ell_r=(\boldsymbol{w}_r-\hat{\boldsymbol{w}}_r)\boldsymbol{h}(\boldsymbol{w}_r-\hat{\boldsymbol{w}}_r)^\top$. For $\lambda=0$, this is twice the squared output-reconstruction error. Quantize columns in a fixed order with optimal compensation of the remaining columns~\citep{gptq}.
Write $\tilde w_{rj}$ for the compensated value of coordinate $j$ at the moment it is quantized, $q_{rj}$ for its reconstruction, and $e_{rj}=\tilde w_{rj}-q_{rj}$.
\begin{enumerate}
\item[(i)] \emph{Exact diagonalization.} The layer loss of row $r$ equals
\[
\ell_r=\sum_{j=1}^n d_j\,e_{rj}^2,\qquad d_j=\frac{1}{[(\boldsymbol{h}_{f_j})^{-1}]_{jj}}>0,
\]
where $f_j$ is the set of columns not yet quantized at step $j$. The sensitivities $d_j$ depend only on $\boldsymbol{h}$ and the order, not on the row.
\item[(ii)] \emph{Exact decomposition.} With $\bar d=\frac1n\sum_j d_j$ and $\overline{e_r^2}=\frac1n\sum_j e_{rj}^2$, as an identity,
\[
\ell_r=\bar d\sum_j e_{rj}^2+n\,\mathrm{Cov}_j\big(d_j,e_{rj}^2\big),\qquad
\mathrm{Cov}_j(\boldsymbol{d},e^2):=\frac1n\sum_j (d_j-\bar d)\big(e_{rj}^2-\overline{e_r^2}\big).
\]
\item[(iii)] \emph{Deterministic bound.} For $\overline{e_r^2}>0$, with $\mathrm{CV}(\boldsymbol{d})=\mathrm{sd}_j(\boldsymbol{d})/\bar d$ and $\mathrm{CV}(e_r^2)=\mathrm{sd}_j(e_r^2)/\overline{e_r^2}$,
\[
\frac{\big|\ell_r-\bar d\sum_j e_{rj}^2\big|}{\bar d\sum_j e_{rj}^2}\;\le\;\mathrm{CV}(\boldsymbol{d})\cdot\mathrm{CV}(e_r^2)=:\varepsilon_r .
\]
If the relative-gap bound holds uniformly with some $\varepsilon<1$ over the candidate level sets, and $m_{\mathrm{unw}}$ minimizes $\sum_j e_{rj}^2$ and $m^\star$ minimizes $\ell_r$ over level sets, then $\ell_r(m_{\mathrm{unw}})\le\frac{1+\varepsilon}{1-\varepsilon}\,\ell_r(m^\star)$.
\end{enumerate}
\end{proposition}
\begin{proof}
(i) For the quadratic loss $(\boldsymbol{w}-\hat{\boldsymbol{w}})\boldsymbol{h}(\boldsymbol{w}-\hat{\boldsymbol{w}})^\top$, quantizing one coordinate $j$ to $q_j$ and re-optimizing all coordinates in $f_j\setminus\{j\}$ increases the loss by exactly $e_j^2/[(\boldsymbol{h}_{f_j})^{-1}]_{jj}$, the optimal-brain-surgeon identity that GPTQ applies column by column~\citep{gptq}.
The increments are additive because each step re-optimizes the remaining free coordinates; summing over $j$ gives the identity.
(ii) $\sum_j d_je_j^2=\sum_j[\bar d+(d_j-\bar d)][\overline{e^2}+(e_j^2-\overline{e^2})]$, and the two cross terms vanish because $\sum_j(d_j-\bar d)=0$ and $\sum_j(e_j^2-\overline{e^2})=0$.
(iii) Cauchy-Schwarz on the covariance gives the bound; the consequence is the sandwich $(1-\varepsilon)\bar d\,u(m)\le \ell(m)\le(1+\varepsilon)\bar d\,u(m)$ for all $m$, with $u=\sum_j e_j^2$.
\end{proof}

The covariance identity isolates the difference between sensitivity-weighted and unweighted sequential residuals.
Smaller sensitivity spread tightens the bound for a given error profile.
Stage 2 fits levels to the transformed weights before sequential compensation, whereas the residuals $e_{rj}$ in this proposition are measured during compensation.
The identity therefore explains the loss diagnostic without asserting that the empirical level fit globally minimizes the compensated layer objective.
GPTAQ additionally accounts for the difference between full-precision and quantized layer inputs~\citep{gptaq}; \cref{app:algorithm} gives its place in the pipeline.

For \cref{tab:cvd-gap}, the Hessian uses 128 sequences of 2048 WikiText2 tokens and damping $\lambda=0.01\,\overline{\mathrm{diag}(2\boldsymbol{x}\boldsymbol{x}^\top)}$.
The reported $d_j$ are the sequential sensitivities in \cref{prop:unweighted}, computed on the remaining coordinates at each step.
Projections sharing an input share its Hessian.
The rotated setting uses the folded pipeline rotation; the \texttt{o} and \texttt{down} projections use a randomized-Hadamard proxy for their online transforms.
The unrotated setting uses the same activations with these transforms removed.
The relative gap is $g_r=\sum_jd_je_{rj}^2/(\bar d\sum_je_{rj}^2)-1$ for the pipeline's W2 sequential residuals.
We report the median and range across the 32 blocks for each projection.

\begin{table}[t]
\caption{Sequential sensitivity spread and weighted-to-unweighted residual gaps on LLaMA-2-7B. Brackets give the range across blocks.}
\label{tab:cvd-gap}
\begin{center}
\small
\setlength{\tabcolsep}{4.5pt}
\begin{tabular}{lrcccc}
\toprule
& & \multicolumn{2}{c}{$\mathrm{CV}(\boldsymbol{d})$} & \multicolumn{2}{c}{$\mathrm{med}_r\,|g_r|$} \\
\cmidrule(lr){3-4}\cmidrule(lr){5-6}
\textbf{Proj.} & $n$ & unrotated & rotated & unrot.\ ($\times10^{-2}$) & rot.\ ($\times10^{-4}$) \\
\midrule
\texttt{q} & 4096 & 0.205 [0.077, 0.278] & 0.015 [0.012, 0.023] & 0.95 [0.44, 67.75] & 3.22 [2.59, 8.15] \\
\texttt{k} & 4096 & 0.205 [0.077, 0.278] & 0.015 [0.012, 0.023] & 2.62 [0.50, 64.68] & 3.25 [2.67, 7.12] \\
\texttt{v} & 4096 & 0.205 [0.077, 0.278] & 0.015 [0.012, 0.023] & 0.36 [0.11, 27.64] & 3.20 [2.65, 4.98] \\
\texttt{o} & 4096 & 0.782 [0.505, 2.717] & 0.017 [0.006, 0.022] & 4.89 [1.94, 94.23] & 3.57 [1.53, 4.71] \\
\texttt{gate} & 4096 & 0.207 [0.078, 0.273] & 0.014 [0.012, 0.020] & 0.57 [0.12, 7.01] & 3.10 [2.69, 4.71] \\
\texttt{up} & 4096 & 0.207 [0.078, 0.273] & 0.014 [0.012, 0.020] & 0.45 [0.10, 5.13] & 3.09 [2.62, 4.54] \\
\texttt{down} & 11008 & 0.793 [0.497, 18.658] & 0.007 [0.001, 0.011] & 1.46 [0.94, 18.19] & 0.93 [0.12, 1.60] \\
\bottomrule
\end{tabular}
\end{center}
\end{table}

Across projection types, the median sensitivity coefficient of variation falls from 0.205 to 0.793 before rotation to 0.007 to 0.017 afterward.
The corresponding median absolute residual gap is approximately $0.9\times10^{-4}$ to $3.6\times10^{-4}$ after rotation.
These measurements support the use of scalar distortion as a practical reconstruction objective in this setting; the deterministic bound itself does not assume a distribution for the covariance term.

\subsection{Exact arithmetic representation}

\begin{proposition}[Losslessness of the sign-amplitude form at $b{=}2$]
\label{prop:lossless}
Consider reconstructions $\hat w=a_1s_1+a_2s_2$ with $s_1,s_2\in\{-1,+1\}$ and $a_1>a_2>0$.
\begin{enumerate}
\item[(i)] The representable set is $\{\pm(a_1+a_2),\pm(a_1-a_2)\}$, and $(m_1,m_2)\mapsto(a_1,a_2)=\big(\tfrac{m_2+m_1}{2},\tfrac{m_2-m_1}{2}\big)$ is a bijection from $\{0<m_1<m_2\}$ onto $\{a_1>a_2>0\}$ with inverse $m_2=a_1+a_2$, $m_1=a_1-a_2$.
Hence the sign-amplitude form represents exactly the antisymmetric zero-free four-point level sets $\{\pm m_1,\pm m_2\}$.
\item[(ii)] Let $p$ be a symmetric density whose four-level MMSE quantizer is unique, as holds for the centered Gaussian reference~\citep{fleischer1964,max1960}.
Then the unconstrained optimum is antisymmetric and zero-free, hence representable, and
\[
\min_{a_1>a_2>0}\ \mathbb{E}_p\big[(w-a_1s_1(w)-a_2s_2(w))^2\big]\;=\;\min_{m\in\mathbb{R}^4}\ \mathbb{E}_p\big[(w-q_m(w))^2\big].
\]
\end{enumerate}
\end{proposition}
\begin{proof}
(i) Enumerating $(s_1,s_2)$ gives the four values; $a_1\pm a_2>0$ if and only if $a_1>a_2>0$; the map and its inverse are as stated, and $m_2>m_1>0$ if and only if $a_2>0$ and $a_1>a_2$.
(ii) If $m^\star$ is optimal then so is $-m^\star$ by symmetry of $p$; uniqueness forces $m^\star=-m^\star$ as a set, that is $m^\star=\{\pm m_1,\pm m_2\}$.
A symmetric set of four distinct points cannot contain $0$, since $0$ would pair with itself, so $m_1>0$; and $m_1\neq m_2$ since the levels are distinct.
Thus $m^\star$ lies in the representable set of (i); the constrained minimum can neither exceed the unconstrained one nor fall below it.
\end{proof}

\begin{corollary}[One normalized grid for every row]
\label{cor:grid}
For $p=\mathcal{N}(0,\sigma^2)$ the unique optimum is $m^\star=\sigma\{\pm0.4528,\pm1.5104\}$ with distortion $d^\star=0.11748\,\sigma^2$~\citep{max1960}, so $a_1=0.9816\,\sigma$ and $a_2=0.5288\,\sigma$.
The minimum-distortion solution is scale-covariant, $m^\star(\sigma z)=\sigma\,m^\star(z)$~\citep{lloyd1982,max1960}; therefore when every rotated row is Gaussian up to scale, all rows share one normalized level set and differ only in $\sigma_r$.
This gives a population reference for Stage 2; the implemented fit uses each transformed row's empirical weights.
\end{corollary}
\begin{remark}
The $\{0,1\}$ basis $\hat w=\alpha_1b_1+\alpha_2b_2$ represents $\{0,\alpha_1,\alpha_2,\alpha_1+\alpha_2\}$, which is neither antisymmetric nor zero-free without an additional shift term; it cannot express the optimum of (ii) natively.
\end{remark}

\section{Quantization procedure and experimental settings}
\label{app:algorithm}
\label{app:protocol}

\Cref{alg:fidelityq} gives the full procedure.
Rotations follow QuaRot~\citep{quarot}, and sequential compensation uses GPTAQ~\citep{gptaq}, extending the Hessian-aware updates of GPTQ~\citep{gptq}.
Reconstruction levels are fitted before compensation and held fixed during code assignment; KIC then updates their magnitudes with signs and assignments frozen.

\begin{algorithm}[htbp]
\caption{\method{} quantization and calibration}
\label{alg:fidelityq}
\begin{algorithmic}[1]
\REQUIRE Full-precision model $f$, calibration pool $\mathcal{D}$, budget $k$, configuration W2A$b_A$KV$b_{KV}$
\ENSURE Quantized model $\hat f$
\STATE Fold the offline rotations into the weights and install the online Hadamard transforms
\FOR{each quantized layer, in forward order}
  \STATE Fit $m_{1,r},m_{2,r}$ to each transformed row under \cref{eq:mmse}
  \STATE With these levels fixed, assign codes sequentially and compensate remaining errors
  \STATE Encode assignments as sign planes and derive $(a_{1,r},a_{2,r})$ using \cref{eq:decomp}
\ENDFOR
\STATE Enable activation and KV-cache quantization at the requested bit widths to obtain $\hat f$
\STATE Compute teacher-student token divergence on the teacher's top-512 support
\STATE Score each sequence using \cref{eq:damage-score} with $\rho=0.25$; retain the top-$k$ set $\mathcal{S}_k$
\STATE Optimize only $\{m_{1,r},m_{2,r}\}_r$ on $\mathcal{S}_k$ under \cref{eq:repair}
\STATE Recompute decoder amplitudes from the calibrated magnitudes and store the packed model
\end{algorithmic}
\end{algorithm}

\subsection{Calibration}

The main downstream checkpoints use 4096 KIC-selected sequences, except LLaMA-2-13B, which uses 1024.
\Cref{tab:kic-settings} lists the current KIC training settings; the 13B run takes 3072 optimizer steps.
The teacher distribution is cached and renormalized on its top-512 tokens, while the soft-target loss uses the student's log probabilities on those token IDs, as in \cref{eq:repair}.

\begin{table}[htbp]
\caption{KIC training settings. Only the two reconstruction magnitudes per row are trainable.}
\label{tab:kic-settings}
\centering
\small
\begin{tabular}{ll}
\toprule
Setting & Value \\
\midrule
Optimizer & Adam \\
Learning rate & $5\times10^{-6}$, constant \\
Epochs / batch size & 3 / one sequence, no accumulation \\
Sequence length & 2048 tokens \\
Loss & Hard CE $+$ $0.1\times$ soft-target CE \\
Teacher support / temperature & Top-512 tokens / $T=1$ \\
Gradient-norm clipping & 1.0 \\
\bottomrule
\end{tabular}
\end{table}

The scalar baselines differ from \method{} in calibration budget: GPTQ and QuaRot follow the GPTQ protocol of 128 calibration sequences without gradient training, whereas \method{} trains the reconstruction levels on the KIC-selected sequences. The budget-comparable reference is Base\,$+$\,IAC in \cref{tab:component-ablation}, which uses no gradient training and reaches 7.78 WikiText2 perplexity on LLaMA-2-7B at W2A16, below every scalar baseline in \cref{tab:main}. AQLM and EfficientQAT also train quantization parameters and are comparable to \method{} in this respect.

The component and representation ablations in \cref{tab:component-ablation,tab:quantizer-ablation} use 512 WikiText2 sequences; their C4 results measure transfer from WikiText2 calibration.
The selection study in \cref{tab:selection} fixes the budget at 512, and the corpus study in \cref{tab:datamix} compares 512 and 1024 sequences.
The separate Hessian diagnostic uses 128 sequences, as specified in \cref{app:theory}.

\subsection{Evaluation}

Perplexity is evaluated at context length 2048.
Every \method{} column in \cref{tab:main} comes from a single checkpoint per model and configuration, calibrated on a mixed-domain pool of KIC-selected sequences, following the calibration practice of \citet{aqlm}.
C4 evaluation samples 256 independent 2048-token spans from validation documents with seed 0; the 13B C4 column uses its segment-sampled results.
Downstream tasks use the LM Evaluation Harness~\citep{evalharness}.
The seven-task mean covers ARC-e, ARC-c, HellaSwag, PIQA, WinoGrande, BoolQ, and OpenBookQA; the five-task mean covers the first five.
HellaSwag uses normalized accuracy, and LAMBADA-Std, LAMBADA-OpenAI, and 5-shot MMLU are reported separately.
Source-specific evaluation differences in the extended comparisons are marked in \cref{app:results}.

\section{Additional ablations}
\label{app:ablations}

\subsection{Level placement and representation}

\Cref{tab:quantizer-ablation} compares adaptive levels with a range-tuned uniform grid and a zero-anchored reconstruction family.
Range tuning preserves the uniform grid's relative spacing; without KIC it gives WikiText2/C4 perplexities of 7.89/12.87, compared with 7.78/12.53 for IAC.
With KIC, the $\{0,1\}$ zero-anchored basis gives lower WikiText2 perplexity but higher C4 perplexity than the sign-plane representation in both precision settings.
The IAC reference rows repeat \cref{tab:component-ablation} under the same 512-sequence WikiText2 calibration protocol.

\begin{table}[!t]
\caption{Level-placement and representation controls on LLaMA-2-7B. Bold marks the best result within each precision setting.}
\label{tab:quantizer-ablation}
\begin{center}
\scriptsize
\renewcommand{\arraystretch}{1.08}
\setlength{\tabcolsep}{2.2pt}
\resizebox{\textwidth}{!}{%
\begin{tabular}{ll|cc|ccccccc|cc|ccc}
\toprule
\multirow{2}{*}{\textbf{Variant}} & \multirow{2}{*}{\textbf{W-A-KV}}
 & \multicolumn{2}{c|}{\textbf{Perplexity}\,$\downarrow$}
 & \multicolumn{2}{c}{\textbf{ARC}} & \multirow{2}{*}{\textbf{HellaS}} & \multirow{2}{*}{\textbf{PIQA}} & \multirow{2}{*}{\textbf{WinoG}} & \multirow{2}{*}{\textbf{BoolQ}} & \multirow{2}{*}{\textbf{OBQA}}
 & \multicolumn{2}{c|}{\textbf{Avg.}\,$\uparrow$}
 & \multicolumn{2}{c}{\textbf{LAMBADA}} & \multirow{2}{*}{\textbf{MMLU}$_{5}$} \\
\cmidrule(lr){3-4}\cmidrule(lr){5-6}\cmidrule(lr){12-13}\cmidrule(lr){14-15}
 &  & Wiki2 & C4 & Easy & Chg &  &  &  &  &  & 7-task & 5-task & Std & OpenAI &  \\
\midrule
Base + IAC & 2-16-16 & 7.78 & 12.53 & 51.43 & 28.84 & 52.98 & 67.03 & 61.33 & 62.35 & 34.2 & 51.17 & 52.32 & 39.14 & 51.64 & 28.73 \\
Clipped uniform (no KIC) & 2-16-16 & 7.89 & 12.87 & 49.71 & 29.01 & 51.62 & 67.08 & 62.12 & 62.32 & 33.6 & 50.78 & 51.91 & 39.39 & 51.39 & 28.67 \\
\addlinespace[3pt]
Base + IAC + KIC & 2-16-16 & 7.11 & \textbf{9.74} & \textbf{60.19} & \textbf{36.35} & \textbf{65.70} & \textbf{74.05} & \textbf{63.93} & 68.69 & \textbf{37.4} & \textbf{58.04} & \textbf{60.04} & \textbf{56.90} & \textbf{67.16} & \textbf{29.54} \\
Zero-anchored + KIC & 2-16-16 & \textbf{6.87} & 12.82 & 57.32 & 35.07 & 65.12 & 73.07 & 61.33 & \textbf{71.04} & 36.8 & 57.11 & 58.38 & 48.19 & 58.84 & 26.50 \\
\midrule
Base + IAC + KIC & 2-4-4 & 7.53 & \textbf{10.20} & 54.46 & 34.47 & \textbf{63.90} & \textbf{71.49} & \textbf{62.27} & 63.98 & 34.4 & \textbf{55.00} & \textbf{57.32} & \textbf{50.28} & \textbf{61.79} & 26.09 \\
Zero-anchored + KIC & 2-4-4 & \textbf{7.03} & 13.98 & \textbf{55.68} & \textbf{34.64} & 62.05 & 70.13 & 60.85 & \textbf{65.08} & \textbf{36.4} & 54.98 & 56.67 & 43.10 & 51.87 & \textbf{26.16} \\
\bottomrule
\end{tabular}}
\end{center}
\end{table}

\subsection{Calibration sequence selection}

\Cref{tab:selection} compares selection rules on Qwen3-8B using a common C4/FineWeb candidate pool~\citep{c4,fineweb}, a 512-sequence budget, and a fixed calibration optimizer.
Random selection uses seed 1234.
The prior-feature control fits ridge regression (regularization 10) on half of the pool to predict exact tail-KL from 14 standardized text and teacher-derived features.
The two-stage cascade shortlists the prior predictor's top half of the pool, then selects 512 sequences by exact tail-KL.
The tail-damage arm scores the full pool using \cref{eq:damage-score}.
The token-weighting control keeps the random subset and weights token cross-entropy by KL divided by its sequence mean, capped at 5.

Enrichment measures selected tail-KL mass relative to a random draw; entropy is relative to the candidate-pool mean teacher entropy.
Exact tail-damage selection gives the lowest perplexity among these rules, reducing WikiText2 perplexity from 17.05 to 16.69 and C4 from 24.46 to 24.37.
This study uses Qwen3-8B; the main table uses Qwen3-8B-Base.

\begin{table}[t]
\caption{Sequence-selection ablation on Qwen3-8B with a fixed 512-sequence calibration budget. Lower perplexity is better.}
\label{tab:selection}
\begin{center}
\small
\setlength{\tabcolsep}{6pt}
\begin{tabular}{lcccc}
\toprule
\textbf{Selection rule} & \textbf{Enrichment} & \textbf{Entropy} & \textbf{Wiki2}\,$\downarrow$ & \textbf{C4}\,$\downarrow$ \\
\midrule
Random & 1.00$\times$ & 1.01$\times$ & 17.05 & 24.46 \\
Prior-feature ridge & 1.20$\times$ & 1.01$\times$ & 16.92 & 24.45 \\
Two-stage cascade & 1.29$\times$ & 0.93$\times$ & 16.88 & 24.47 \\
Tail-damage score & 1.45$\times$ & 0.89$\times$ & \textbf{16.69} & \textbf{24.37} \\
\bottomrule
\end{tabular}
\end{center}
\end{table}

\subsection{Calibration corpus and budget}

\Cref{tab:datamix} compares calibration corpora and budgets on Qwen3-8B.
At 512 sequences, FineWeb gives the lowest perplexity on both evaluations, followed by the C4/FineWeb mixture.
Increasing the budget to 1024 improves both C4-only and mixed-corpus calibration: the mixture gives lower WikiText2 perplexity, while C4-only gives lower C4 perplexity.

\begin{table}[t]
\caption{Calibration corpus and budget ablation on Qwen3-8B. Rows use 512 sequences unless stated otherwise.}
\label{tab:datamix}
\begin{center}
\small
\begin{tabular}{lcc}
\toprule
\textbf{Calibration data} & \textbf{Wiki2} & \textbf{C4} \\
\midrule
C4 only & 17.18 & 24.43 \\
FineWeb only & 16.75 & 24.32 \\
C4 + FineWeb & 16.83 & 24.33 \\
C4 only, 1024 sequences & 16.44 & \textbf{23.34} \\
mix, 1024 sequences & \textbf{16.26} & 23.44 \\
\bottomrule
\end{tabular}
\end{center}
\end{table}

\section{Arithmetic decoding and deployment}
\label{app:kernel}

\subsection{Fused arithmetic decoding}
\label{app:kernel:impl}

At batch size 1, a persistent CUDA kernel computes the two sign-plane reductions and combines their row amplitudes:
\[
\boldsymbol{y}=\operatorname{diag}(\boldsymbol{a}_1)\boldsymbol{s}_1\boldsymbol{x}
+\operatorname{diag}(\boldsymbol{a}_2)\boldsymbol{s}_2\boldsymbol{x}.
\]
Packed signs act directly on fp16 activation pairs through XOR before accumulation.
This avoids per-weight conversion, indexed reconstruction lookup, and materialization of the decoded weight matrix.
Activations are staged in shared memory, and pairwise fp16 accumulators are periodically accumulated into fp32.
In a CUDA-event comparison of linear kernels, this implementation takes 5.18\,ms/token versus 8.29 for our generic Triton path, a $1.60\times$ speedup for the same representation.

\subsection{Measurement protocol}
\label{app:kernel:env}

Deployment measurements use 16-bit activations on an NVIDIA RTX A6000-48Q vGPU (48\,GB, \texttt{sm\_86}); \method{} uses W2A16.
The persistent kernel uses a launch-configuration search; comparison paths use the released QuIP\#, AQLM, BitBLAS, and Marlin operators.
The environment is CUDA 12.1, PyTorch 2.4.0+cu121, Triton 3.0.0, Transformers 4.51.3, AQLM 1.1.7, BitBLAS 0.1.0.post1, Marlin 0.1.1, and Fast Hadamard Transform 1.1.0.

The linear-stack benchmark uses 20 warm-up calls, followed by 200 consecutive launches and a single synchronization; wall-clock time is divided by 200.
CUDA graphs are disabled.
For LLaMA-2-7B, timings sum 32 blocks, each with four $4096\times4096$, two $11008\times4096$, and one $4096\times11008$ linear layer.
Figure~\ref{fig:tradeoff}a reports kernel time for every method, so no path carries an added transform cost. For \method{} on LLaMA-2-7B the online Hadamard transforms at the \texttt{o\_proj} and \texttt{down\_proj} inputs~\citep{quarot} add 1.11\,ms to the 5.11\,ms kernel time.
Baselines use their measured operator paths without additional estimated transform costs.
QuIP\# perplexity follows the 2048-token reevaluation of \citet{efficientqat}; AQLM uses $2\times8$ as in \cref{tab:main}.
Hollow points time surrogate decoding paths implemented by us: uniform BitBLAS for GPTQ, two sign planes for DB-LLM~\citep{dbllm}, lookup for LeanQuant/GPTVQ~\citep{leanquant,gptvq}, and dense-plus-sparse for PB-LLM~\citep{pbllm}.

\subsection{Storage and full-model generation}
\label{app:kernel:prefill}

\method{} stores two sign bits per weight and two fp16 decoder amplitudes per output row, giving $2+32/d_i$ bits per linear-layer weight, or approximately 2.007 bits averaged over LLaMA-2-7B.
\Cref{tab:e2e} counts packed linear weights and reconstruction metadata, excluding embeddings and the output head for every method.
Generation uses a 128-token prompt followed by 256 greedily generated tokens in HuggingFace eager mode; reported timings are medians of three runs and include the online transforms.
AQLM uses $1\times16$ in this table and $2\times8$ in \cref{fig:tradeoff}a.

\begin{table}[t]
\caption{Linear-stack kernel time per token at batch size 1 on an RTX A6000 at W2A16,
with WikiText2 perplexity on LLaMA-2-7B repeated from \cref{tab:main}. Times exclude the
online rotations, which add 1.11\,ms for \method{} on LLaMA-2-7B. Scalar decoding paths
are listed first. A dash marks a configuration we did not time.}
\label{tab:latency}
\begin{center}
\small
\setlength{\tabcolsep}{5pt}
\begin{tabular}{llccc}
\toprule
\textbf{Method} & \textbf{Decoding path} & \textbf{LLaMA-2-7B} & \textbf{Qwen3-8B} & \textbf{Wiki2}\,$\downarrow$ \\
 & & (ms) & (ms) & \\
\midrule
FP16 & \texttt{torch} linear & 25.26 & 27.01 & 5.47 \\
\midrule
QuaRot & BitBLAS uint2 per-row & 4.72 & 5.01 & 24.52 \\
ResComp & BitBLAS uint2 per-row & 4.72 & 5.01 & 10.28 \\
Base (uniform grid) & BitBLAS uint2 per-row & 4.72 & 5.01 & 9.56 \\
DB-LLM & additive decomposition & 5.11 & --- & 8.89 \\
\textbf{\method{}} & fused sign planes & \textbf{5.11} & \textbf{5.42} & \textbf{6.62} \\
\midrule
QuIP\# E8P & \texttt{quiptools\_cuda} & 8.83 & 11.43 & 6.66 \\
AQLM $2\times8$ & AQLM CUDA operator & 9.93 & --- & 7.24 \\
\bottomrule
\end{tabular}
\end{center}
\end{table}

\begin{table}[t]
\caption{LLaMA-2-7B generation in HuggingFace eager mode at batch size 1. Storage covers linear weights and reconstruction metadata.}
\label{tab:e2e}
\begin{center}
\small
\setlength{\tabcolsep}{6pt}
\begin{tabular}{lccc}
\toprule
\textbf{Method} & \textbf{Weights (GB)} & \textbf{Prefill (ms/token)} & \textbf{Decode (tokens/s)} \\
\midrule
FP16 & 12.06 & 0.27 & 35.9 \\
Marlin W4A16 & 3.02 & 0.21 & 39.8 \\
BitBLAS uniform W2 & 1.51 & 6.07 & 35.6 \\
QuIP\# E8P & 1.51 & 1.12 & 13.8 \\
AQLM 1$\times$16 & 1.73 & 1.98 & 23.2 \\
\textbf{\method{}} & 1.51 & 7.32 & \textbf{47.0} \\
\bottomrule
\end{tabular}
\end{center}
\end{table}

The prefill fallback reconstructs each layer's fp16 weights before GEMM and releases them afterward, giving 7.32\,ms per prompt token.
Using unrounded rates, $128t_{\mathrm{prefill}}+256/r_{\mathrm{decode}}$ gives 6.4\,s for \method{} and 7.2\,s for FP16 for this workload.
\Cref{fig:tradeoff}b reports storage relative to FP16 and throughput relative to QuIP\#.

\section{Extended results}
\label{app:results}
\label{app:llama3-results}

\Cref{tab:main-full,tab:main-full-13b,tab:main-full-cont,tab:main-full-qwen,tab:main-full-30b} provide full task results and additional comparisons, including LLaMA-3-8B.
W-A-KV denotes nominal precision, with $g64$ and $g128$ indicating grouped weight scales.
Results follow the evaluation protocol in \cref{app:protocol}.

In these tables, $^{*}$ marks our reproduction and $^{m}$ our evaluation of a released checkpoint.
Additional baselines include GPTQ~\citep{gptq}, EfficientQAT~\citep{efficientqat}, QuIP\#~\citep{quip-sharp}, and AutoRound~\citep{autoround}; AQLM variants are identified by their codebook configuration~\citep{aqlm}.
QuIP\# perplexities marked $^{e}$ follow the reevaluation of \citet{efficientqat}.
The RCP WikiText2 result quoted in the main text, 8.31 for LLaMA-2-7B at W2A4KV4, is from Table~2 of \citet{rcp}.

% Full task results for the main comparisons and complete supplementary baselines.
\begin{table}[p]
\caption{Extended results on LLaMA-2-7B. Markers are defined in \cref{app:results}.}
\label{tab:main-full}
\begin{center}
\scriptsize
\setlength{\tabcolsep}{2.2pt}
\resizebox{\textwidth}{!}{%
\begin{tabular}{ll|cc|ccccccc|cc|ccc}
\toprule
\multirow{2}{*}{\textbf{Method}} & \multirow{2}{*}{\textbf{W-A-KV}}
 & \multicolumn{2}{c|}{\textbf{Perplexity}\,$\downarrow$}
 & \multicolumn{2}{c}{\textbf{ARC}} & \multirow{2}{*}{\textbf{HellaS}} & \multirow{2}{*}{\textbf{PIQA}} & \multirow{2}{*}{\textbf{WinoG}} & \multirow{2}{*}{\textbf{BoolQ}} & \multirow{2}{*}{\textbf{OBQA}}
 & \multicolumn{2}{c|}{\textbf{Avg.}\,$\uparrow$}
 & \multicolumn{2}{c}{\textbf{LAMBADA}} & \multirow{2}{*}{\textbf{MMLU}$_5$} \\
\cmidrule(lr){3-4}\cmidrule(lr){5-6}\cmidrule(lr){12-13}\cmidrule(lr){14-15}
 & & Wiki2 & C4 & Easy & Chg & & & & & & 7-task & 5-task & Std & OpenAI & \\
\midrule
FP16 & 16-16-16 & 5.47 & 6.97 & 73.86 & 45.05 & 76.16 & 78.78 & 69.46 & 79.36 & 44.0 & 66.67 & 68.66 & 68.81 & 73.65 & 45.94 \\
RTN$^{*}$ & 2-16-16 & \multicolumn{2}{c|}{\emph{diverged}} & 26.22 & 27.05 & 25.57 & 50.54 & 51.30 & 37.98 & 24.8 & 34.78 & 36.14 & 0.00 & 0.00 & 24.80 \\
RTN & $2_{g64}$-16-16 & 431.97 & --- & --- & --- & --- & --- & --- & --- & --- & --- & --- & --- & --- & --- \\
GPTQ$^{*}$ & 2-16-16 & 29.83 & 335.50 & 41.84 & 27.30 & 42.38 & 62.35 & 52.09 & 56.61 & 29.0 & 44.51 & 45.19 & 21.99 & 25.58 & 22.95 \\
QuaRot$^{*}$ & 2-16-16 & 24.52 & 39.39 & 39.35 & 23.72 & 33.30 & 58.98 & 52.17 & 61.41 & 28.2 & 42.45 & 41.50 & 15.72 & 17.04 & 25.87 \\
ResComp$^{*}$ & $2_{g128}$-16-16 & 9.14 & 16.95 & 51.35 & 28.58 & 47.64 & 66.00 & 59.27 & 63.06 & 30.8 & 49.53 & 50.57 & 28.64 & 39.34 & 27.47 \\
ResComp$^{*}$ & 2-16-16 & 10.28 & 19.68 & 45.96 & 26.71 & 44.15 & 64.85 & 58.25 & 63.00 & 31.0 & 47.70 & 47.98 & 24.92 & 33.07 & 27.17 \\
DB-LLM$^{*}$ & 2-16-16 & 8.89 & 19.85 & 42.80 & 27.90 & 51.99 & 62.13 & 55.96 & 61.19 & 32.6 & 47.80 & 48.16 & 32.19 & 45.47 & 24.95 \\
EfficientQAT\,(QAT)$^{m}$ & $2_{g128}$-16-16 & 7.16 & 8.90 & 66.25 & 36.60 & 67.43 & 75.03 & 65.51 & 72.02 & 39.0 & 60.26 & 62.16 & 59.19 & 67.32 & 33.55 \\
EfficientQAT\,(QAT)$^{m}$ & $2_{g64}$-16-16 & 6.86 & 8.51 & 64.06 & 38.05 & 68.88 & 74.86 & 64.64 & 76.12 & 39.0 & 60.80 & 62.10 & 57.91 & 66.68 & 33.04 \\
\textbf{\method{}} & 2-16-16 & 6.62 & 9.53 & 63.80 & 35.67 & 68.33 & 74.97 & 66.14 & 68.10 & 40.4 & 59.63 & 61.78 & 61.89 & 68.87 & 37.29 \\
QuIP\#$^{*}$ & 2-16-16 & 6.66 & 8.35 & 68.14 & 40.53 & 69.79 & 76.55 & 66.38 & 71.35 & 41.4 & 62.02 & 64.28 & 61.58 & 70.41 & 37.06 \\
AQLM ($2\times8$)$^{m}$ & 2-16-16 & 7.24 & 8.96 & 60.27 & 36.77 & 66.74 & 74.10 & 62.51 & 69.27 & 36.6 & 58.04 & 60.08 & 56.26 & 67.34 & 33.03 \\
Base (uniform grid) & 2-16-16 & 9.56 & 18.43 & 49.24 & 26.71 & 44.50 & 66.16 & 57.93 & 62.29 & 29.2 & 48.00 & 48.91 & 24.45 & 32.37 & 26.70 \\
\midrule
QuaRot$^{*}$ & 2-4-4 & 47.23 & 77.49 & 32.07 & 23.12 & 30.02 & 53.54 & 50.43 & 58.41 & 25.8 & 39.06 & 37.84 & 6.04 & 8.00 & 24.92 \\
ResComp$^{*}$ & $2_{g128}$-4-4 & 11.06 & 20.85 & 45.12 & 26.88 & 41.14 & 61.15 & 54.38 & 62.63 & 31.0 & 46.04 & 45.73 & 18.40 & 26.26 & 24.90 \\
ResComp$^{*}$ & 2-4-4 & 12.78 & 23.91 & 42.93 & 26.62 & 40.19 & 61.81 & 52.41 & 61.68 & 28.0 & 44.81 & 44.79 & 18.77 & 25.46 & 26.01 \\
\textbf{\method{}} & 2-4-4 & 6.96 & 9.97 & 63.13 & 35.67 & 65.26 & 73.78 & 62.04 & 64.89 & 36.0 & 57.25 & 59.98 & 57.44 & 65.01 & 33.99 \\
Base (uniform grid) & 2-4-4 & 11.27 & 21.41 & 43.22 & 25.09 & 39.59 & 62.51 & 54.30 & 61.99 & 28.4 & 45.01 & 44.94 & 15.18 & 23.35 & 25.55 \\
\bottomrule
\end{tabular}}
\end{center}
\end{table}

\begin{table}[p]
\caption{Extended results on LLaMA-2-13B. Markers are defined in \cref{app:results}.}
\label{tab:main-full-13b}
\begin{center}
\scriptsize
\setlength{\tabcolsep}{2.2pt}
\resizebox{\textwidth}{!}{%
\begin{tabular}{ll|cc|ccccccc|cc|ccc}
\toprule
\multirow{2}{*}{\textbf{Method}} & \multirow{2}{*}{\textbf{W-A-KV}}
 & \multicolumn{2}{c|}{\textbf{Perplexity}\,$\downarrow$}
 & \multicolumn{2}{c}{\textbf{ARC}} & \multirow{2}{*}{\textbf{HellaS}} & \multirow{2}{*}{\textbf{PIQA}} & \multirow{2}{*}{\textbf{WinoG}} & \multirow{2}{*}{\textbf{BoolQ}} & \multirow{2}{*}{\textbf{OBQA}}
 & \multicolumn{2}{c|}{\textbf{Avg.}\,$\uparrow$}
 & \multicolumn{2}{c}{\textbf{LAMBADA}} & \multirow{2}{*}{\textbf{MMLU}$_5$} \\
\cmidrule(lr){3-4}\cmidrule(lr){5-6}\cmidrule(lr){12-13}\cmidrule(lr){14-15}
 & & Wiki2 & C4 & Easy & Chg & & & & & & 7-task & 5-task & Std & OpenAI & \\
\midrule
FP16 & 16-16-16 & 4.88 & 6.47 & 76.47 & 48.81 & 79.62 & 80.25 & 72.45 & 82.11 & 45.2 & 69.27 & 71.52 & 70.68 & 76.54 & 55.43 \\
RTN$^{*}$ & 2-16-16 & 5634.55 & 4815.95 & 26.01 & 28.16 & 26.20 & 49.35 & 48.07 & 37.83 & 23.4 & 34.15 & 35.56 & 0.00 & 0.00 & 22.92 \\
GPTQ$^{*}$ & 2-16-16 & 38.55 & 108.93 & 28.87 & 24.83 & 28.61 & 51.36 & 50.36 & 57.22 & 25.4 & 38.09 & 36.81 & 0.39 & 8.00 & 23.28 \\
QuaRot$^{*}$ & 2-16-16 & 12.17 & 24.11 & 43.81 & 25.09 & 45.14 & 62.30 & 56.59 & 62.54 & 28.8 & 46.32 & 46.59 & 31.11 & 40.42 & 25.44 \\
ResComp$^{*}$ & 2-16-16 & 7.90 & 14.05 & 55.68 & 30.20 & 53.24 & 68.66 & 62.83 & 63.18 & 33.4 & 52.46 & 54.12 & 38.95 & 50.01 & 27.97 \\
EfficientQAT\,(QAT)$^{m}$ & $2_{g128}$-16-16 & 6.08 & 7.75 & 72.43 & 44.28 & 74.78 & 77.64 & 69.30 & 78.07 & 43.2 & 65.67 & 67.69 & 63.19 & 71.49 & 46.54 \\
EfficientQAT\,(QAT)$^{m}$ & $2_{g64}$-16-16 & 5.96 & 7.59 & 70.20 & 44.11 & 74.95 & 77.97 & 69.38 & 78.10 & 42.6 & 65.33 & 67.32 & 65.69 & 72.42 & 46.29 \\
Base (uniform grid) & 2-16-16 & 7.78 & 14.43 & 54.55 & 30.63 & 50.14 & 67.63 & 62.75 & 62.48 & 33.2 & 51.63 & 53.14 & 34.80 & 48.40 & 27.38 \\
\textbf{\method{}} & 2-16-16 & 5.86 & 8.43 & 69.19 & 41.47 & 74.39 & 77.53 & 69.14 & 75.26 & 40.8 & 63.97 & 66.34 & 62.90 & 71.43 & 45.54 \\
QuIP\#$^{*}$ & 2-16-16 & 5.74$^{e}$ & 7.45$^{e}$ & 71.68 & 43.94 & 75.33 & 78.73 & 70.24 & 78.41 & 43.8 & 66.02 & 67.98 & 67.42 & 74.23 & 50.19 \\
AQLM ($1\times16$)$^{m}$ & 2-16-16 & 5.59 & 7.27 & 74.66 & 47.18 & 76.89 & 78.89 & 70.40 & 79.36 & 43.2 & 67.23 & 69.60 & 66.97 & 74.79 & 49.77 \\
AQLM ($2\times8$)$^{m}$ & 2-16-16 & 6.06 & 7.80 & 69.57 & 42.49 & 72.68 & 77.58 & 64.80 & 75.84 & 39.8 & 63.25 & 65.42 & 63.30 & 72.00 & 45.38 \\
DB-LLM$^{*}$ & 2-16-16 & 7.97 & 17.61 & 55.39 & 33.19 & 59.40 & 68.77 & 59.67 & 65.44 & 34.4 & 53.75 & 55.28 & 41.63 & 51.14 & 26.62 \\
\midrule
QuaRot$^{*}$ & 2-4-4 & 92387.88 & 95839.34 & 25.76 & 28.50 & 25.96 & 50.00 & 51.78 & 37.83 & 27.2 & 35.29 & 36.40 & 0.00 & 0.00 &  \\
Base (uniform grid) & 2-4-4 & 8.91 & 16.54 & 48.44 & 28.67 & 45.55 & 65.61 & 57.85 & 62.14 & 29.8 & 48.29 & 49.22 & 24.80 & 37.96 & 25.96 \\
\textbf{\method{}} & 2-4-4 & 6.04 & 8.81 & 68.31 & 39.51 & 72.64 & 76.28 & 65.43 & 73.43 & 39.4 & 62.14 & 64.43 & 61.44 & 69.88 & 43.74 \\
ResComp$^{*}$ & 2-4-4 & 9.70 & 19.10 & 48.23 & 25.94 & 43.38 & 64.53 & 59.51 & 62.17 & 31.4 & 47.88 & 48.32 & 26.41 & 40.60 & 27.19 \\
\bottomrule
\end{tabular}}
\end{center}
\end{table}

\begin{table}[p]
\caption{Extended results on LLaMA-3-8B. Markers are defined in \cref{app:results}.}
\label{tab:main-full-cont}
\label{tab:llama3-results}
\begin{center}
\scriptsize
\setlength{\tabcolsep}{2.2pt}
\resizebox{\textwidth}{!}{%
\begin{tabular}{ll|cc|ccccccc|cc|ccc}
\toprule
\multirow{2}{*}{\textbf{Method}} & \multirow{2}{*}{\textbf{W-A-KV}}
 & \multicolumn{2}{c|}{\textbf{Perplexity}\,$\downarrow$}
 & \multicolumn{2}{c}{\textbf{ARC}} & \multirow{2}{*}{\textbf{HellaS}} & \multirow{2}{*}{\textbf{PIQA}} & \multirow{2}{*}{\textbf{WinoG}} & \multirow{2}{*}{\textbf{BoolQ}} & \multirow{2}{*}{\textbf{OBQA}}
 & \multicolumn{2}{c|}{\textbf{Avg.}\,$\uparrow$}
 & \multicolumn{2}{c}{\textbf{LAMBADA}} & \multirow{2}{*}{\textbf{MMLU}$_5$} \\
\cmidrule(lr){3-4}\cmidrule(lr){5-6}\cmidrule(lr){12-13}\cmidrule(lr){14-15}
 & & Wiki2 & C4 & Easy & Chg & & & & & & 7-task & 5-task & Std & OpenAI & \\
\midrule
FP16 & 16-16-16 & 6.14 & 8.88 & 77.74 & 54.10 & 79.22 & 80.58 & 73.16 & 81.83 & 44.8 & 70.20 & 72.96 & 68.81 & 75.59 & 65.48 \\
DB-LLM$^{m}$ & $2_{g128}$-16-16 & 13.62 & 19.26 & 52.15 & 29.52 & 53.55 & 69.15 & 60.46 & 66.24 & 34.0 & 52.15 & 52.97 & 38.37 & 50.36 & 28.03 \\
DB-LLM$^{*}$ & 2-16-16 & 14.70 & 55.66 & 25.38 & 24.15 & 27.14 & 50.92 & 50.51 & 39.20 & 27.4 & 34.96 & 35.62 & 0.12 & 0.47 & 23.69 \\
QuaRot$^{*}$ & 2-16-16 & 22.90 & 45.88 & 27.40 & 24.23 & 27.59 & 52.50 & 49.25 & 39.60 & 25.4 & 35.14 & 36.19 & 1.26 & 3.03 & 25.37 \\
ResComp$^{*}$ & 2-16-16 & 15.17 & 37.52 & 36.24 & 24.57 & 39.06 & 58.11 & 56.43 & 62.26 & 29.2 & 43.70 & 42.88 & 19.31 & 26.99 & 25.64 \\
EfficientQAT\,(QAT)$^{m}$ & $2_{g128}$-16-16 & 9.79 & 13.23 & 67.34 & 41.13 & 68.94 & 75.95 & 66.93 & 75.29 & 39.6 & 62.17 & 64.06 & 58.02 & 65.77 & 45.61 \\
EfficientQAT\,(QAT)$^{m}$ & $2_{g64}$-16-16 & 9.40 & 12.77 & 70.16 & 43.09 & 70.42 & 76.39 & 67.09 & 76.36 & 41.6 & 63.59 & 65.43 & 59.29 & 67.77 & 47.29 \\
\textbf{\method{}} & 2-16-16 & 9.63 & 15.08 & 62.71 & 37.46 & 66.56 & 73.45 & 66.93 & 68.72 & 38.0 & 59.12 & 61.42 & 53.64 & 61.98 & 35.29 \\
AQLM$^{m}$ & 2-16-16 & 7.76 & 10.89 & 73.06 & 47.53 & 74.51 & 78.73 & 71.03 & 77.03 & 41.6 & 66.21 & 68.97 & 64.43 & 71.92 & 56.38 \\
\midrule
QuaRot$^{*}$ & 2-4-4 & 805.62 & 929.30 & 26.22 & 25.00 & 26.14 & 50.54 & 49.41 & 37.92 & 27.6 & 34.69 & 35.46 & 0.00 & 0.14 & 25.00 \\
ResComp$^{*}$ & $2_{g128}$-4-4 & 18.04 & 41.08 & 36.20 & 22.70 & 34.74 & 58.11 & 53.04 & 59.02 & 27.4 & 41.60 & 40.96 & 13.68 & 17.83 & 25.84 \\
ResComp$^{*}$ & 2-4-4 & 19.95 & 53.11 & 35.23 & 23.12 & 32.69 & 56.26 & 50.91 & 59.36 & 27.6 & 40.74 & 39.64 & 12.15 & 13.60 & 25.89 \\
\textbf{\method{}} & 2-4-4 & 10.76 & 16.49 & 54.97 & 31.06 & 61.49 & 70.73 & 61.48 & 68.07 & 36.2 & 54.86 & 55.95 & 48.17 & 57.00 & 30.53 \\
\bottomrule
\end{tabular}}
\end{center}
\end{table}

\begin{table}[p]
\caption{Extended results on Qwen3-8B-Base. Markers are defined in \cref{app:results}.}
\label{tab:main-full-qwen}
\begin{center}
\scriptsize
\setlength{\tabcolsep}{2.2pt}
\resizebox{\textwidth}{!}{%
\begin{tabular}{ll|cc|ccccccc|cc|ccc}
\toprule
\multirow{2}{*}{\textbf{Method}} & \multirow{2}{*}{\textbf{W-A-KV}}
 & \multicolumn{2}{c|}{\textbf{Perplexity}\,$\downarrow$}
 & \multicolumn{2}{c}{\textbf{ARC}} & \multirow{2}{*}{\textbf{HellaS}} & \multirow{2}{*}{\textbf{PIQA}} & \multirow{2}{*}{\textbf{WinoG}} & \multirow{2}{*}{\textbf{BoolQ}} & \multirow{2}{*}{\textbf{OBQA}}
 & \multicolumn{2}{c|}{\textbf{Avg.}\,$\uparrow$}
 & \multicolumn{2}{c}{\textbf{LAMBADA}} & \multirow{2}{*}{\textbf{MMLU}$_5$} \\
\cmidrule(lr){3-4}\cmidrule(lr){5-6}\cmidrule(lr){12-13}\cmidrule(lr){14-15}
 & & Wiki2 & C4 & Easy & Chg & & & & & & 7-task & 5-task & Std & OpenAI & \\
\midrule
FP16 & 16-16-16 & 7.00 & 10.42 & 79.88 & 56.23 & 78.60 & 79.54 & 72.06 & 83.18 & 41.8 & 70.18 & 73.26 & 66.02 & 70.83 & 76.84 \\
AutoRound$^{*}$ & 2-16-16 & 87.25 & 64.30 & 46.04 & 29.27 & 42.28 & 62.30 & 55.56 & 59.36 & 31.4 & 46.60 & 47.09 & 4.72 & 6.85 & 23.19 \\
QuaRot$^{*}$ & 2-16-16 & 12.58 & 23.02 & 61.15 & 34.47 & 52.47 & 67.79 & 61.88 & 63.67 & 31.6 & 53.29 & 55.55 & 43.78 & 49.56 & 43.93 \\
DB-LLM$^{*}$ & 2-16-16 & 11.62 & 23.33 & 50.42 & 30.97 & 50.92 & 64.64 & 57.46 & 58.90 & 35.6 & 49.84 & 50.88 & 32.18 & 40.68 & 27.55 \\
\textbf{\method{}} & 2-16-16 & 8.84 & 13.64 & 71.93 & 44.88 & 70.28 & 74.97 & 67.96 & 74.92 & 42.0 & 63.85 & 66.00 & 57.91 & 65.11 & 59.94 \\
ResComp$^{*}$ & 2-16-16 & 11.21 & 22.82 & 54.88 & 32.00 & 49.11 & 65.51 & 61.25 & 66.79 & 31.4 & 51.56 & 52.55 & 39.90 & 43.84 & 36.06 \\
Base (uniform grid) & 2-16-16 & 11.55 & 25.78 & 53.54 & 30.97 & 45.20 & 64.42 & 58.48 & 62.57 & 30.8 & 49.43 & 50.52 & 40.52 & 40.33 & 30.69 \\
\midrule
QuaRot$^{*}$ & 2-4-4 & 20.38 & 40.50 & 48.74 & 28.75 & 43.89 & 60.66 & 53.12 & 62.02 & 29.0 & 46.60 & 47.03 & 30.06 & 32.49 & 30.52 \\
ResComp$^{*}$ & $2_{g128}$-4-4 & 13.47 & 25.73 & 53.03 & 32.76 & 46.60 & 64.09 & 56.91 & 64.01 & 32.6 & 50.00 & 50.68 & 31.30 & 36.77 & 32.43 \\
ResComp$^{*}$ & 2-4-4 & 14.85 & 31.40 & 52.82 & 29.18 & 43.53 & 63.33 & 54.22 & 62.48 & 28.0 & 47.65 & 48.62 & 30.49 & 35.69 & 30.86 \\
\textbf{\method{}} & 2-4-4 & 10.69 & 16.59 & 68.10 & 42.92 & 62.86 & 72.09 & 60.22 & 68.99 & 39.0 & 59.17 & 61.24 & 49.62 & 54.76 & 47.60 \\
Base (uniform grid) & 2-4-4 & 15.56 & 36.47 & 43.94 & 26.96 & 39.39 & 60.45 & 52.72 & 62.54 & 29.4 & 45.06 & 44.69 & 22.10 & 25.60 & 28.30 \\
\bottomrule
\end{tabular}}
\end{center}
\end{table}

\begin{table}[p]
\caption{Extended results on LLaMA-30B. Markers are defined in \cref{app:results}.}
\label{tab:main-full-30b}
\begin{center}
\scriptsize
\setlength{\tabcolsep}{2.2pt}
\resizebox{\textwidth}{!}{%
\begin{tabular}{ll|cc|ccccccc|cc|ccc}
\toprule
\multirow{2}{*}{\textbf{Method}} & \multirow{2}{*}{\textbf{W-A-KV}}
 & \multicolumn{2}{c|}{\textbf{Perplexity}\,$\downarrow$}
 & \multicolumn{2}{c}{\textbf{ARC}} & \multirow{2}{*}{\textbf{HellaS}} & \multirow{2}{*}{\textbf{PIQA}} & \multirow{2}{*}{\textbf{WinoG}} & \multirow{2}{*}{\textbf{BoolQ}} & \multirow{2}{*}{\textbf{OBQA}}
 & \multicolumn{2}{c|}{\textbf{Avg.}\,$\uparrow$}
 & \multicolumn{2}{c}{\textbf{LAMBADA}} & \multirow{2}{*}{\textbf{MMLU}$_5$} \\
\cmidrule(lr){3-4}\cmidrule(lr){5-6}\cmidrule(lr){12-13}\cmidrule(lr){14-15}
 & & Wiki2 & C4 & Easy & Chg & & & & & & 7-task & 5-task & Std & OpenAI & \\
\midrule
FP16 & 16-16-16 & 4.10 & 5.98 & 79.25 & 53.58 & 83.01 & 81.83 & 75.85 & 83.79 & 49.2 & 72.36 & 74.70 & 73.06 & 77.61 & 58.30 \\
\hdashline[1.5pt/1.5pt]
ResComp & 2-16-16 & 6.45 & 10.38 & 68.01 & 38.05 & 63.98 & 74.32 & 69.30 & 68.75 & 38.0 & 60.06 & 62.73 & 54.57 & 67.20 & 32.00  \\
Base (uniform grid) & 2-16-16 & 6.33 & 10.10 & 67.05 & 38.57 & 64.91 & 74.70 & 68.59 & 66.54 & 39.0 & 59.91 & 62.76 & 54.63 & 68.33 & 31.29  \\
\textbf{\method{}} & 2-16-16 & \textbf{5.14} & \textbf{7.64} & \textbf{72.10} & \textbf{45.56} & \textbf{79.23} & \textbf{78.67} & \textbf{71.98} & \textbf{79.45} & \textbf{42.4} & \textbf{67.06} & \textbf{69.51} & \textbf{68.21} & \textbf{74.83} & \textbf{49.96}  \\
\hdashline[1.5pt/1.5pt]
QuaRot & 2-4-4 & 99.92 & 961.80 & 50.46 & 29.95 & 46.67 & 61.04 & 57.38 & 61.04 & 31.6 & 48.31 & 49.10 & 21.42 & 22.74 & 25.99  \\
ResComp & 2-4-4 & 6.88 & 11.14 & 59.72 & 36.01 & 61.29 & 72.85 & 67.48 & 63.00 & 36.0 & 56.62 & 59.47 & 48.44 & 62.22 & 29.67  \\
Base (uniform grid) & 2-4-4 & 6.86 & 11.04 & 61.45 & 35.32 & 60.44 & 71.44 & 66.77 & 63.61 & 37.2 & 56.60 & 59.08 & 48.38 & 61.91 & 28.42  \\
\textbf{\method{}} & 2-4-4 & \textbf{5.71} & \textbf{7.79} & \textbf{71.84} & \textbf{44.28} & \textbf{76.29} & \textbf{77.31} & \textbf{71.19} & \textbf{73.06} & \textbf{43.2} & \textbf{65.31} & \textbf{68.18} & \textbf{63.01} & \textbf{71.92} & \textbf{38.64}  \\
\bottomrule
\end{tabular}}
\end{center}
\end{table}

\end{document}